\documentclass{article}
\usepackage[preprint]{icml2026}

\usepackage[utf8]{inputenc}
\usepackage{amsmath,amssymb,amsthm,amsfonts}
\usepackage{algorithm}
\usepackage{algorithmic}
\usepackage{graphicx}
\usepackage{wrapfig}
\usepackage{subcaption}
\usepackage{booktabs}
\usepackage{float}
\usepackage{placeins}
\usepackage{natbib}
\usepackage{enumitem}
\usepackage{mathtools}
\usepackage{tikz}
\usepackage{pgfplots}
\usepackage{fontspec}
\usepackage[english,provide=*]{babel}
\usepackage{changepage}
\usepackage{pgfplots}
\usepgfplotslibrary{groupplots}
\usepackage{pgfplotstable}
\usepackage{bm}

\usepackage{hyperref}
\usepackage{url}
\usepackage{natbib}

\pgfplotsset{compat=1.18}

\newtheorem{theorem}{Theorem}
\newtheorem{proposition}{Proposition}
\newtheorem{lemma}{Lemma}
\newtheorem{remark}{Remark}

\babelprovide[import,onchar=ids fonts]{english}
\babelprovide[onchar=ids fonts]{tifinagh}
\babelfont[tifinagh]{rm}[Path=fonts/]{ebrima.ttf}

\newfontfamily\tifinaghfont[Path=fonts/]{NotoSansTifinagh-Regular.ttf}
\DeclareRobustCommand{\E}{\text{\normalfont\tifinaghfont ⵟ}}

\newcommand{\R}{\mathbb{R}}
\newcommand{\Sph}{\mathbb{S}}

\begin{document}

\twocolumn[
\icmltitle{SLAY:\\
Geometry-Aware Spherical Linearized Attention with Yat-Kernel}
\icmltitlerunning{SLAY: Geometry-Aware Spherical Linearized Attention with Yat-Kernel}

\icmlsetsymbol{equalcontrib}{*}
\icmlsetsymbol{advisor}{†}

\begin{icmlauthorlist}
\icmlauthor{Jose Miguel Luna}{columb,azetta,equalcontrib}
\icmlauthor{Taha Bouhsine}{azetta,equalcontrib}
\icmlauthor{Krzysztof Choromanski}{columb,deepmind,advisor}
\end{icmlauthorlist}

\icmlaffiliation{columb}{Columbia University}
\icmlaffiliation{azetta}{Azetta.ai}
\icmlaffiliation{deepmind}{Google DeepMind}
\icmlcorrespondingauthor{Jose Miguel Luna}{jml2404@columbia.edu}

\icmlkeywords{Linear Attention, Neural Matter Networks, Random Features, Long Context}

\vskip 0.3in
]

\printAffiliationsAndNotice{\icmlEqualContribution $^\dagger$Senior Advisor.}

\begin{abstract}
We propose a new class of linear‑time attention mechanisms based on a relaxed and computationally efficient formulation of the recently introduced \emph{$\mathbb{E}$‑Product}, often referred to as the Yat‑kernel \citep{nomoredelulu}. The resulting interactions are geometry‑aware and inspired by inverse‑square interactions in physics. 
Our method, \emph{Spherical Linearized Attention with Yat Kernels} \textbf{(\emph{SLAY})}, constrains queries and keys to the unit sphere so that attention depends only on angular alignment. Using Bernstein’s theorem, we express the spherical Yat‑kernel as a nonnegative mixture of polynomial–exponential product kernels and derive a strictly positive random‑feature approximation enabling linear‑time $O(L)$ attention. We establish positive definiteness and boundedness on the sphere and show that the estimator yields well‑defined, nonnegative attention scores.
Empirically, SLAY achieves performance that is nearly indistinguishable from standard softmax attention while retaining linear time and memory scaling, and consistently outperforms prior linear‑time attention mechanisms such as Performers and Cosformers. To the best of our knowledge, SLAY represents the closest linear‑time approximation to softmax attention reported to date, enabling scalable Transformers without the typical performance trade‑offs of attention linearization.
\end{abstract}

\section{Introduction \& Related Work}
\label{sec:intro}

Transformer models \citep{transformers-base, vit, bert, brown, longt5, bigbird} owe much of their success to the \textit{attention mechanism}, which enables dynamic, content-dependent interactions between tokens. In standard Transformers, attention is implemented via softmax operation applied to pairwise query--key similarities. While expressive, it requires constructing explicit $L \times L$ attention matrices for $L$-length input sequences, resulting in quadratic in $L$ time and space complexity. This cost rapidly becomes prohibitive as context lengths grow, fundamentally limiting long-context modeling \citep{yaofu}.

To overcome this bottleneck several so-called \textit{efficient-attention} mechanisms, characterized by time complexity sub-quadratic (often linear) in $L$ were proposed.
They leverage a large spectrum of techniques ranging from clustering/hashing \citep{reformer, routing, lsh-transformer, clustr} to approaches casting attention matrix as a kernel matrix and applying techniques based on random features (RFs) \citep{performers, performer2, chefs, dense-exponential, luo, fft-t, arora}. Those mechanisms critically rely on \textbf{positive} RFs, since regular trigonometric mechanisms, that can be derived from the seminal papers on linearizing shift-invariant kernels \citep{random-features, sinks} lead to unstable training.

Despite these advances, softmax attention remains tied to a specific similarity: the exponential of an inner product. This choice conflates alignment and magnitude, and its unbounded growth requires careful normalization and stabilization. These limitations motivate alternative attention kernels that are geometrically grounded, self-regularizing, and compatible with efficient computation, as explored in activation-free architectures \citep{nomoredelulu}.

\textit{Neural Matter Networks (NMNs)} introduced the \emph{\E-Product}, also referred to as the Yat-kernel \citep{nomoredelulu}, a kernel operator inspired by inverse-square interactions in physics:
\begin{equation}
\E(\mathbf{q},\mathbf{k}) \overset{\mathrm{def}}{=} \frac{(\mathbf{q}^\top \mathbf{k})^2}{\|\mathbf{q}-\mathbf{k}\|^2 + \epsilon}
\label{eq:eproduct}
\end{equation}
where $\epsilon > 0$ ensures numerical stability. Unlike standard dot-product similarity, the \E-Product explicitly couples two geometric quantities: \emph{alignment} and \emph{proximity}. The squared inner product in the numerator yields an even (sign-symmetric) alignment score, while the inverse-distance denominator penalizes interactions between distant vectors. This ratio yields a self-regularizing response that can suppress irrelevant interactions without requiring explicit activation functions or normalization layers.

From a theoretical perspective, the \E-Product can be viewed as a kernel-like similarity operator. The NMNs, constructed as linear combinations of \E-Product units, are universal approximators on compact domains, despite being entirely activation-free \citep{nomoredelulu}. In the context of attention, the \E-Product offers a geometry-aware alternative to softmax similarity. It favors tokens that are both aligned and close in representation space. However, the same geometric coupling introduces a computational obstacle: the denominator $\|\mathbf{q}-\mathbf{k}\|_{2}^2 = \|\mathbf{q}\|_{2}^2 + \|\mathbf{k}\|_{2}^2 - 2\mathbf{q}^\top \mathbf{k}$ entangles query and key terms and prevents the factorization required for efficient linear attention. As a result, naive \E-Product attention still requires explicit pairwise interactions and reverts to quadratic complexity. This mirrors the general limitation for non-factorizable kernels that motivates linearization techniques such as those used in Performers and related models \citep{performers, hedgehog, insuhan}.

This paper addresses this limitation for \E-Attention. We show that by constraining queries and keys to lie on the unit sphere and reformulating the resulting kernel using Bernstein’s Theorem, the \E-Product admits a nonnegative mixture representation, involving polynomial--exponential product kernels, that can be approximated by strictly positive random features. This yields a linear-time attention mechanism that preserves the core geometric and self-regulating properties of the \E-Product while enabling scalable long-context Transformers. We refer to our proposed mechanism as: \emph{Spherical Linearized Attention with YAT-Kernels}, or \emph{SLAY}. Empirically, SLAY matches full spherical YAT attention and compares favorably to linearized baselines on regular Transformer tasks, while retaining $O(L)$-scaling. To summarize, SLAY:
\vspace{-2mm}
\begin{itemize}[leftmargin=1.4em]
\item Enforces unit-norm constraints on queries and keys, decoupling alignment and distance.
\vspace{-1mm}
\item Linearizes the spherical \E-product via an integral representation based on Bernstein’s Theorem.
\vspace{-1mm}
\item Approximates the resulting kernel using strictly positive Tensor Product Random Features.
\vspace{-1mm}
\item Achieves linear-time $O(L)$ attention while preserving key theoretical properties of NMNs.
\end{itemize}
\vspace{-2mm}
This paper is organized as follows: (1) in Section \ref{sec:method}, we present SLAY and provide detailed corresponding theoretical insight, (2) in Section \ref{sec:experiments} we test SLAY on various tasks, including language and vision, by comparing quality-wise with the brute-force mechanisms, as well as other efficient attention methods. We also present speed tests. We conclude in Section \ref{sec:conclusion} and provide additional results in the Appendix.

\section{SLAY Mechanism}
\label{sec:method}

\subsection{Preliminaries}
\label{sec:preliminaries}
Denote by $\mathbf{Q} \in \mathbb{R}^{L \times d_{QK}}, \mathbf{K} \in \mathbb{R}^{L \times d_{QK}}$ the query- and key-matrix respectively, with rows denoted as follows: $\mathbf{q}_{1},...,\mathbf{q}_{L} \in \mathbb{R}^{d_{QK}}$ and $\mathbf{k}_{1},...,\mathbf{k}_{L} \in \mathbf{R}^{d_{QK}}$. The attention matrix $\mathbf{A}_{\E} \in \mathbb{R}^{L \times L}$ considered in this paper is of the form: $\mathbf{A}_{\E}=[\E(\mathbf{\widehat{q}}_{i},\mathbf{\widehat{k}}_{j})]_{i=1,...,L}^{j=1,...,L}$, where $\widehat{\mathbf{x}}$ denotes the $L_{2}$-normalized version of $\mathbf{x}$. Our goal is to compute the attention action leveraging $\mathbf{A}_{\E}$, that we refer to as \E-attention, with the spherical constraints, without forming the $L\times L$ matrix $\mathbf{A}_{\E}$ of pairwise interactions. We proceed in several steps, summarized below:
\begin{itemize}
\item linearization of the non-factorizable terms induced by denominators in the YAT-kernel formula via a Laplace integral (Bernstein's Theorem),
\item discretization of the resulting nonnegative mixture using Gauss--Laguerre quadrature,
\item approximating resulting kernels with positive random features from \citep{performers},
\item applying standard matrix multiplication re-ordering for the calculations involving obtained low-rank factorized attention matrix.
\end{itemize}


\subsection{Spherical Constraint}

We assume $d \ge 2$ throughout, where spherical isotropy theory applies. We normalize inputs as follows:
\begin{equation}
\widehat{\mathbf{q}} = \frac{\mathbf{q}}{\|\mathbf{q}\|_{2}},
\widehat{\mathbf{k}} = \frac{\mathbf{q}}{\|\mathbf{k}\|_{2}},
\|\widehat{\mathbf{q}}\|_{2}=\|\widehat{\mathbf{k}}\|_{2}=1
\end{equation}
Expanding the denominator of Eq.~\eqref{eq:eproduct} yields:
\begin{align}
\|\widehat{\mathbf{q}}-\widehat{\mathbf{k}}\|_{2}^2 + \epsilon
&= \|\widehat{\mathbf{q}}\|_{2}^2 + \|\widehat{\mathbf{k}}\|_{2}^2 - 2\widehat{\mathbf{q}}^\top \widehat{\mathbf{k}} + \epsilon \\
&= (2+\epsilon) - 2\widehat{\mathbf{q}}^\top \widehat{\mathbf{k}}.
\end{align}
Let $x=\widehat{\mathbf{q}}^\top \widehat{\mathbf{k}}\in[-1,1]$ and $C=2+\epsilon$. After input normalization, the spherical \E-product becomes:
\begin{equation}
\E_{\text{sph}}(\widehat{\mathbf{q}},\widehat{\mathbf{{k}}})
= \frac{x^2}{C-2x}.
\label{eq:spherical}
\end{equation}
Thus, the kernel depends only on angular alignment.

\paragraph{Geometric intuition.}
On the unit sphere $\Sph^{d-1}$, the squared \emph{chordal} distance is defined as follows:
\[
d_{\Sph^{d-1}}(\widehat{\mathbf{q}}, \widehat{\mathbf{k}})^2 = 2(1 - \widehat{\mathbf{q}}^\top \widehat{\mathbf{k}}).
\]
We conclude that the spherical \E-product can be written as a distance-regularized alignment score:
\begin{equation}
\E_{\mathrm{sph}}(\hat q, \hat k) = \frac{(\widehat{\mathbf{q}}^{\top}\widehat{\mathbf{k}})^{2}}{d_{\Sph^{d-1}}(\widehat{\mathbf{q}}, \widehat{\mathbf{k}})^2 + \epsilon}.
\label{eq:geometric}
\end{equation}
Additional discussion (including invariances and the connection to isotropic spherical kernels \citep{schoenberg1942}) is deferred to Appendix~\ref{app:geom}.

\subsection{Integral Linearization}
With spherical constraints in place, our next step is to linearize the factor $\frac{1}{C-2x}$ in the Eq. \ref{eq:spherical} for the spherical \E-product. We will use Bernstein's Theorem for that.

The function $g(y) = 1/y$ is completely monotone on $(0,\infty)$, which by Bernstein's Theorem implies the Laplace representation $1/y = \int_0^\infty e^{-sy}\,ds$ \citep{widder1941laplace,schilling2012}. To apply this identity to our kernel, we substitute $y = C - 2x$ and verify that $y > 0$ for all $x \in [-1,1]$. Since $x \le 1$ and $C=2+\epsilon$, we have $y=C-2x \ge \epsilon > 0$, hence Bernstein's representation applies (see Lemma~\ref{lem:complete-monotone} in Appendix~\ref{app:background}).

We conclude that by invoking the Laplace representation, the spherical \E-product can be re-written as:
\begin{align}
\E_{\text{sph}}(\widehat{\mathbf{q}},\widehat{\mathbf{k}})
&= x^2 \int_0^\infty e^{-s(C-2x)}ds \\
&= \int_0^\infty e^{-sC}\,\Big[x^2 e^{2s x}\Big]\,ds.
\label{eq:integral}
\end{align}
This expresses the spherical \E-product as a positively weighted mixture of product kernels:
namely, a degree-2 polynomial factor $(\widehat{\mathbf{q}}^\top \widehat{\mathbf{k}})^2$ multiplied by an exponential dot-product kernel $e^{2s\,\widehat{\mathbf{q}}^\top \widehat{\mathbf{k}}}$.
Importantly, the factor $x^2$ cannot be absorbed into a \emph{nonnegative} Laplace weight over plain exponentials without introducing signed correction terms (see Appendix~\ref{app:features}).

\subsection{Quadratures and Positive Features}

\subsubsection{Quadrature (Gauss--Laguerre)}

We approximate the integral in Eq.~\eqref{eq:integral} using $R$-point Gauss--Laguerre quadrature. We apply Gauss--Laguerre quadrature after the change of variables $t = Cs$, so that $\int_0^\infty e^{-Cs} h(s)\,ds = \frac{1}{C}\int_0^\infty e^{-t} h(t/C)\,dt$:
\[
\int_0^\infty e^{-sC} h(s)\,ds \;\approx\; \sum_{r=1}^R w_r\, h(s_r),
\]
where $\{t_r,\alpha_r\}_{r=1}^R$ are the standard Gauss--Laguerre nodes and weights for $\int_0^\infty e^{-t}f(t)\,dt$ and
\[
s_r=\frac{t_r}{C},\qquad w_r=\frac{\alpha_r}{C}.
\]
Thus, the $w_r$ already incorporate the $1/C$ factor induced by $t=Cs$. Since $|x| \le 1$, the integrand $h(s) = x^2 e^{2sx}$ is entire and uniformly bounded by $e^{2s}$, ensuring uniform exponential convergence over $x \in [-1,1]$. Such bounds follow from classical results on Gauss--Laguerre quadrature for entire functions of exponential type (see Theorem 3.6.24 in \citep{davis1984,gautschi2004}); for a shapes/implementation summary, see Appendix~\ref{app:impl} and Appendix~\ref{app:quadrature-details}.

\subsubsection{Polynomial and Exponential RFs}
\label{sec:poly-exp-rfs}

\paragraph{Tensor sketches \& beyond for polynomial kernels:}
For the exact polynomial kernel $(\widehat{\mathbf{q}}^\top \widehat{\mathbf{k}})^2$, the explicit feature map $\phi_{\text{poly}}(\mathbf{u}) = \mathrm{vec}(\mathbf{u} \mathbf{u}^\top) \in \R^{d^2}$ yields exact reconstruction:
\[
\langle \phi_{\text{poly}}(\widehat{\mathbf{q}}), \phi_{\text{poly}}(\widehat{\mathbf{k}}) \rangle = (\widehat{\mathbf{q}}^\top \widehat{\mathbf{k}})^2.
\]
Note that vectors produced by $\phi_{\mathrm{poly}}$ maps are $d^{2}$-dimensional. In practice, we reduce dimensionality via standard low-dimensional reduction techniques. We consider TensorSketch \citep{pham2013} as well as three common approximations to the degree-2 polynomial kernel $k_{\text{poly}}(x,y)=(x^\top y)^2$, described formally in Appendix~\ref{app:prelim-poly}: (i) Random Maclaurin (RM) features \citep{karkarnick2012}, (ii) Nystrom features using anchor points \citep{williams2001nystrom}, and (iii) anchor features using squared inner products to fixed anchors.

\paragraph{Anchor features \citep{scholkopf2002learning}.}
Let anchors $\{\mathbf{a}_i\}_{i=1}^P\subset\R^d$ be fixed reference vectors. Anchor features define mapping $\phi$ as follows:
\[
\phi_{\mathrm{anc}}(\mathbf{x}) = \frac{1}{\sqrt{P}}\bigl[(\mathbf{x}^\top \mathbf{a}_i)^2\bigr]_{i=1}^P,
\]
yielding a simple low-rank approximation whose induced inner products are nonnegative. Unlike Nystrom approximations, anchor features do not require inversion/whitening of the anchor Gram matrix and therefore preserve non-negativity of approximate kernel evaluations. Furthermore, they are empirically stable at small feature budgets, and computationally simple ($O(dP)$ per token). Thus, unless stated otherwise, we use \emph{anchor features} as the default polynomial approximation method in SLAY. The multi-scale sweep in Table~\ref{tab:poly-sweep} (Appendix~\ref{app:poly-ablation}) supports this choice.

Table~\ref{tab:poly-approximations} summarizes the trade-offs between various techniques for approximating squared dot-product kernels, highlighting positivity preservation as a key distinction, determining the robustness of the technique. 

\begin{table*}[t]
\centering
\small
\caption{Polynomial kernel approximation options for the kernel $(\mathbf{x}^\top \mathbf{y})^2$. Here $D_p$ denotes the polynomial-feature dimension, and $P$ denotes the number of anchors (when applicable). \emph{Feature cost} is the asymptotic cost of computing the polynomial features for one vector and excludes quadrature/PRF computation, tensor-product fusion/sketching, and the linear-attention contractions; these additional costs drive end-to-end latency in Section~\ref{sec:experiments}.}
\label{tab:poly-approximations}
\begin{tabular*}{\textwidth}{@{\extracolsep{\fill}}lcccc}
\toprule
Method & Dim. & Feature cost & Unbiased? & $\langle\phi(x),\phi(y)\rangle\ge 0$? \\
\midrule
Exact $\mathrm{vec}(uu^\top)$ & $d^2$ & $O(d^2)$ & Yes & Yes \\
TensorSketch \citep{pham2013} & $D_p$ & $\approx O(d + D_p\log D_p)$ & Approx. & No (not guaranteed) \\
Random Maclaurin \citep{karkarnick2012} & $D_p$ & $O(d\,D_p)$ & Yes & No (not guaranteed) \\
Nystrom \citep{williams2001nystrom} & $P$ & $O(d\,P)$ & Approx. & No (not guaranteed) \\
Anchor features (low-rank) \citep{scholkopf2002learning} & $P$ & $O(d\,P)$ & No & Yes \\
\bottomrule
\end{tabular*}
\end{table*}

For the theoretical non-negativity and denominator-positivity guarantees stated later, we require a polynomial component whose induced score estimates are nonnegative (e.g., the exact map, or anchor features). Signed polynomial approximations (TensorSketch, Random Maclaurin) and Nystrom features can yield negative approximate inner products (see Appendix~\ref{app:stability-experiments}) and are therefore treated as accuracy/efficiency baselines rather than positivity-guaranteeing estimators.


\paragraph{Positive Random Features (PRFs).}
To handle the exponential term $e^{2s\,\widehat{\mathbf{q}}^\top \widehat{\mathbf{k}}}$ from Eq. \ref{eq:integral}, we use PRFs for the exponential kernel, proposed in \citep{performers}, applied to the \emph{original normalized vectors} (we refer to them simply as PRFs; note that this is a different mechanisms than features used to approximate squared dot-product kernels, discussed above):
\begin{equation}
\phi_{\text{PRF}}(\mathbf{u}; s)
=
\frac{1}{\sqrt{D}}
\left[
\exp\!\left(\sqrt{2s}\,\omega_i^\top \mathbf{u} - s\right)
\right]_{i=1}^D,
\end{equation}
where $\omega_i \sim \mathcal{N}(0, \mathbf{I}_d)$ are drawn independently.
This construction satisfies:
\[
\mathbb{E}\!\left[
\phi^{\top}_{\text{PRF}}(\hat q; s)
\phi_{\text{PRF}}(\hat k; s) 
\right]
=
e^{2s\,\widehat{\mathbf{q}}^\top \widehat{\mathbf{k}}}
\]
for unit-norm inputs (a standard unbiasedness result for positive random features; see \citep{performers} and Prop.~\ref{prop:prf-unbiased} in Appendix~\ref{app:background}).

\subsubsection{Fusing RFs for Polynomial \& Exponential Kernels}
Since polynomial feature maps and PRFs are applied to vectors (not scalars), we obtain for each scale $r$:
\begin{equation}
\widetilde{\Psi}_r(\mathbf{u})
=
\sqrt{w_r}\,\mathcal{S}\Bigl(\phi_{\text{poly}}(\mathbf{u})\otimes \phi_{\text{PRF}}(\mathbf{u}; s_r)\Bigr),
\end{equation}
where $\mathcal{S}:\R^{D_pD_r}\to\R^{D_t}$ is a (randomized) sketching operator that approximates the tensor-product feature map without explicitly materializing the $D_pD_r$-dimensional Kronecker vector.
We then define $\widetilde{\Psi}(\mathbf{u})$ as the concatenation over $r=1,\dots,R$ (where $R$ stands for the number of quadrature nodes). Conceptually, the target kernel at each quadrature node is the product kernel $(\widehat{\mathbf{q}}^\top \widehat{\mathbf{k}})^2\,e^{2s_r\widehat{\mathbf{q}}^\top \widehat{\mathbf{k}}}$, whose RKHS is the tensor product $\mathcal{H}_{\text{poly}}\otimes\mathcal{H}_{\exp,s_r}$. The sketching operator $\mathcal{S}$ provides a computationally efficient approximation of this tensor-product feature map.

\begin{remark}[Feature Map Target]
\label{rem:feature-target}
$\widetilde{\Psi}$ targets the \emph{integrand} $k_s(x)=x^2 e^{2sx}$; consequently,
\vspace{-2mm}
\[
\mathbb{E}[\langle \widetilde{\Psi}(\hat q), \widetilde{\Psi}(\hat k) \rangle] \approx \sum_{r=1}^R w_r\,(\hat q^\top \hat k)^2 e^{2s_r\hat q^\top \hat k},
\]
which is a quadrature approximation to Eq.~\eqref{eq:integral}.
\end{remark}

\begin{remark}[Bias Decomposition]
\label{rem:bias}
$\langle \widetilde{\Psi}(\hat q), \widetilde{\Psi}(\hat k) \rangle$ is unbiased for the discretized (quadrature) kernel, but biased for the true kernel unless $R\to\infty$.
\end{remark}
\vspace{-3.5mm}
\subsection{Linear-Time Attention Computation}

Given query/key tensors $\mathbf{Q},\mathbf{K} \in \R^{L\times d_{QK}}$ and value tensor $\mathbf{V} \in \mathbb{R}^{L \times d_{V}}$, we compute normalized inputs, apply $\Psi$, and use the standard linear-attention reordering of the computations to calculate attention:
\vspace{-2.5mm}
\begin{equation}
\label{eq:lin-att}
\widehat{\mathbf{Y}} = \frac{\widetilde{\Psi}(\mathbf{Q})\big(\widetilde{\Psi}(\mathbf{K})^\top \mathbf{V}\big)}
{\widetilde{\Psi}(\mathbf{Q})\big(\widetilde{\Psi}(\mathbf{K})^\top \mathbf{1}\big)}.
\end{equation}
Note that this normalization is not a softmax; it corresponds to kernel normalization and preserves linear-time computation.
Here the division is applied row-wise (broadcast across the value dimension); in practice we add a small stabilizer $\delta>0$ to the denominator for numerical stability \citep{higham2002}. Concretely, if $\Psi(\mathbf{Q}),\Psi(\mathbf{K})\in\R^{L\times m}$ (feature dimension $m$) and $\mathbf{V}\in\R^{L\times d_{V}}$, then $\Psi(\mathbf{K})^\top \mathbf{V}\in\R^{m\times d_{V}}$, $\Psi(\mathbf{K})^\top\mathbf{1}\in\R^{m\times 1}$, and the denominator is an $L\times 1$ vector broadcast over $d_{V}$. 
Thus attention matrix $\mathbf{A}_{\E}$ is never explicitly materialized and quadratic in $L$ computations are avoided.
The pseudo-code summarizing the forward pass of the SLAY algorithm is presented in Algorithm 1 Box.

\begin{algorithm}[t]
\caption{Spherical \E-Attention Forward Pass}
\begin{algorithmic}[1]
\REQUIRE $\mathbf{Q},\mathbf{K}\in\R^{L\times d_{QK}}, \mathbf{V} \in \mathbb{R}^{L \times d_{V}}$
\STATE Normalize $\mathbf{Q},\mathbf{K}$ to have unit-norm rows.
\STATE Compute polynomial features $\phi_{\text{poly}}(\mathbf{Q}),\phi_{\text{poly}}(\mathbf{K})$ (e.g., anchor features by default).
\FOR{$r=1$ to $R$}
\STATE Compute PRF features $\phi_{\text{PRF}}(\cdot; s_r)$
\STATE Fuse features via sketched tensor product $\widetilde{\Psi}_r$
\ENDFOR
\STATE Concatenate features $\widetilde{\Psi}(\mathbf{Q}),\widetilde{\Psi}(\mathbf{K})$.
\STATE Compute numerator/denominator as in Eq. \ref{eq:lin-att}.
\RETURN $\widehat{\mathbf{Y}}$
\end{algorithmic}
\end{algorithm}
\paragraph{Complexity Analysis.}
Let $m$ denote the final number of features used in the attention linearization (after concatenating across $R$ quadrature nodes), which in our construction scales as $m=O(RD_t)$. The linear-attention contractions costs $O(L\,m\,d_{V})$ time and $O(L\,m)$ space (the $L\times L$ attention matrix is never formed).
Feature construction depends on the polynomial approximation: exact degree-2 features cost $O(L\,d^2)$, anchor features cost $O(L\,dP)$, Random Maclaurin costs $O(L\,dD_p)$, and TensorSketch costs approximately $O(L\,(d + D_p\log D_p))$ per layer.
The PRF features contribute $O(L\,R\,dD)$ time.
See Appendix~\ref{app:impl} for explicit tensor-product scaling (without sketching) and causal-vs.-non-causal implementation notes.
\vspace{-3mm}
\section{Experiments}
\label{sec:experiments}
We evaluate SLAY along several axes:
\vspace{-3mm}
\begin{itemize}
\item First, we analyze polynomial approximation options to identify optimal configurations for SLAY (see: Sec. \ref{subsec:poly-ablation} for detailed studies).
\vspace{-2mm}
\item Second, we benchmark computational costs by measuring latency, memory, and throughput scaling with sequence length (see: Sec. \ref{subsec:computational-costs}).
\vspace{-2mm}
\item  Third, we evaluate SLAY on \textbf{22} synthetic tasks spanning core capabilities, as well as higher-level behaviors (see: Sec. \ref{subsec:synthetic-experiments}). 
\vspace{-2mm}
\item Fourth, we conduct extreme classification tests to evaluate performance of SLAY versus other attention mechanisms in large-scale classification tasks (see: Sec. \ref{subsec:extreme-classification}). 
\vspace{-2mm}
\item Finally, we test the SLAY mechanism at scale by training and evaluating a series of full-scale Transformer models, varying only the attention mechanism used, leaving every other hyperparameter constant. (see: Sec. \ref{subsec:slayformer}).
\end{itemize}
\vspace{-3mm}
\subsection{Polynomial Factor Approximation}
\label{subsec:poly-ablation}
We isolate the polynomial factor $(\widehat{\mathbf{q}}^\top \widehat{\mathbf{k}})^2$ from the formula of the spherical \E-product (see: Sec. \ref{sec:poly-exp-rfs}) and compare several approximations (anchor, Nystrom, TensorSketch, Random Maclaurin). We report (i) kernel-normalized attention output error relative to exact spherical \E-attention and (ii) forward-pass latency, under matched feature budgets; protocol details and the multi-scale sweep are in Appendix~\ref{app:experiments} and Table~\ref{tab:poly-sweep} (Appendix~\ref{app:poly-ablation}).

For reference, we also include Laplace-only and a Hadamard-fusion variant. These are not polynomial approximations of $(\widehat{\mathbf{q}}^\top \widehat{\mathbf{k}})^2$: they change the estimator by removing or altering the polynomial factor. Under feature-budget matching, they might require substantially larger number of RFs to reach comparable overall feature counts, which can dominate runtime.
\begin{table}[htbp]
\centering
\small
\caption{\small{Kernel approximation quality and latency comparison. Rel.~$\ell_2$ is relative L2 error, Cos is cosine similarity, MSE is mean squared error, and Latency is forward-pass time in milliseconds. For large-scale snapshot, see the ``Large'' block in Table~\ref{tab:poly-sweep}).}}
\label{tab:kernel-combined}
\resizebox{1.05\columnwidth}{!}{%
\begin{tabular}{lcccc}
\toprule
Method & Rel.~$\ell_2\downarrow$ & Cos$\uparrow$ & MSE$\downarrow$ & Latency(ms)$\downarrow$ \\
\midrule
Exact (Spherical) & 0.789 & 0.732 & 1.02e-02 & 5.02 \\
\textbf{Anchor} & \textbf{0.527} & \textbf{0.850} & \textbf{4.55e-03} & \textbf{489.42} \\
Laplace-only & 0.544 & 0.839 & 4.84e-03 & 1905.80 \\
Hadamard (shared $\omega$) & 0.789 & 0.732 & 1.02e-02 & 1932.07 \\
Nystrom & 70.291 & -0.009 & 8.08e+01 & 569.64 \\
TensorSketch & 24823.685 & 0.002 & 1.01e+07 & 547.76 \\
Random Maclaurin & 15826.841 & 0.004 & 4.10e+06 & 551.43 \\
\bottomrule
\end{tabular}%
}
\end{table}

As shown in (Table~\ref{tab:kernel-combined}), anchor and Laplace-only achieve the lowest errors, but anchor is markedly faster (\,$\mathbf{489}$ms vs.~$\mathbf{1906}$ms\,) and thus becomes our default choice in the remaining experiments. Nystrom is substantially less accurate at these budgets. Signed polynomial approximations (TensorSketch and Random Maclaurin) can yield negative approximate inner products, leading to denominator cancellation and severe instability; we include them as efficiency baselines rather than positivity-guaranteeing estimators.

Finally, figure \ref{fig:decision-boundaries} demonstrates the approximation quality of the SLAY anchor kernel, where we observe that SLAY approximates the behavior of the YAT and Spherical YAT kernels remarkably well and at least as well if not better than alternative kernels and their linearized counterparts.

\begin{figure}[htbp]
\centering
\includegraphics[width=\columnwidth]{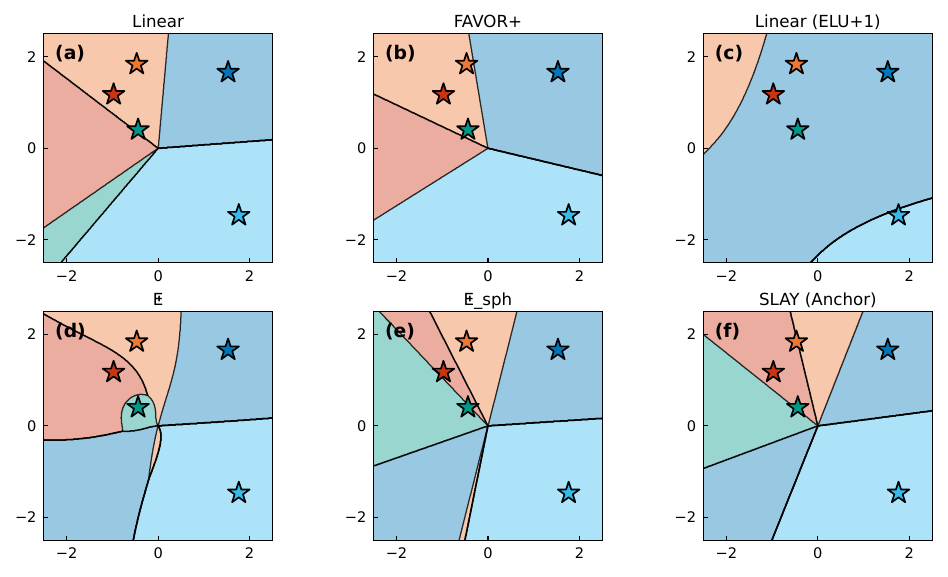}
\caption{Each panel shows how a kernel partitions the 2D feature space among 5 randomly placed neurons (stars). \textbf{(a)} Linear dot product softmax. \textbf{(b)} FAVOR+ (ReLU random features). \textbf{(c)} Linear with ELU+1. \textbf{(d)} Exact \E-kernel. \textbf{(e)} Spherical \E-kernel. \textbf{(f)} SLAY (Anchor)}
\label{fig:decision-boundaries}
\end{figure}

\subsection{Computational Costs}
\label{subsec:computational-costs}

We compare the computational efficiency of \textbf{SLAY} as a linearized estimator of spherical YAT attention, against alternative quadratic and linear attention mechanisms. We report latency, peak memory usage, and throughput as functions of sequence length $L$, focusing on regimes relevant to long-context modeling.

\paragraph{Benchmark setup.}
All attention mechanisms are benchmarked in isolation, using a causal attention kernel with identical architectural settings (embedding dimension $256$, $8$ heads, batch size $1$). Experiments are conducted on a single NVIDIA A100-SXM4 GPU (80\,GB). For each sequence length, we report mean latency over multiple timed runs after warm-up, peak GPU memory allocation, and effective throughput measured in tokens per second. Sequence lengths range from 128 tokens up to 131K tokens, or until out-of-memory (OOM) failure.

\begin{figure}[t!]
\centering
\caption{Attention mechanisms scaling behaviors. Several variants are compared: brute-force regular attention (\textrm{Standard}), YAT-attention (YAT), {linear attention} (\textrm{ELU+1}), Cosformer attention from \citep{zhen2022cosformer}, Performer's attention from \citep{performers} (FAVOR+),  and SLAY's attention.}
\label{fig:scaling}
\vspace{0.3em}

\begin{tikzpicture}

\begin{axis}[
    hide axis,
    xmin=0, xmax=1,
    ymin=0, ymax=1,
    width=\columnwidth*.8,
    height=0.38\columnwidth,
    scale only axis,
    legend columns=3,
    legend style={
        font=\scriptsize,
        at={(0.5,1.03)}, 
        anchor=south,
        draw=none,
        column sep=0.5em,
        /tikz/every even column/.append style={column sep=0.8em}
    },
]
\addlegendimage{black, thick, mark=*, mark size=1.5pt}
\addlegendentry{Standard}
\addlegendimage{black, dashed, thick, mark=square*, mark size=1.5pt}
\addlegendentry{YAT}
\addlegendimage{blue, mark=triangle*, mark size=1.5pt}
\addlegendentry{Linear(ELU+1)}
\addlegendimage{orange, mark=triangle*, mark size=1.5pt}
\addlegendentry{Cosformer}
\addlegendimage{purple, mark=x, mark size=2pt}
\addlegendentry{FAVOR+}
\addlegendimage{red, thick, mark=star, mark size=2pt}
\addlegendentry{SLAY}
\end{axis}

\begin{groupplot}[
    group style={
        group size=1 by 3,
        vertical sep=0.5cm,
    },
    width=\columnwidth*.8,
    height=0.38\columnwidth,
    xmode=log,
    log basis x=2,
    scale only axis,
    grid=both,
    grid style={line width=0.3pt, gray!30},
    major grid style={line width=0.4pt, gray!50},
    tick label style={font=\scriptsize},
    label style={font=\scriptsize},
    every axis plot/.append style={line width=0.6pt, mark size=1.2pt},
]

\nextgroupplot[
    ylabel={Latency (ms)},
    xticklabels={},
    ymin=0,
]

\addplot+[black, thick, mark=*] coordinates {
(128,0.41) (256,0.43) (512,0.45) (1024,0.67)
(2048,2.08) (4096,7.04) (8192,22.84) (16384,85.60)
};
\addplot+[black, dashed, thick, mark=square*] coordinates {
(128,0.57) (256,0.56) (512,0.63) (1024,0.90)
(2048,2.79) (4096,9.79) (8192,37.38) (16384,141.43)
};
\addplot+[blue, mark=triangle*] coordinates {
(128,0.58) (256,0.58) (512,0.72) (1024,1.18)
(2048,2.10) (4096,3.91) (8192,8.42) (16384,18.82)
(32768,36.68) (65536,73.27) (131072,147.63)
};
\addplot+[orange, mark=triangle*] coordinates {
(128,0.96) (256,0.98) (512,1.29) (1024,2.20)
(2048,4.01) (4096,7.52) (8192,14.84) (16384,35.32)
(32768,72.19) (65536,144.60) (131072,291.17)
};
\addplot+[purple, mark=x] coordinates {
(128,0.63) (256,0.64) (512,0.99) (1024,1.55)
(2048,2.69) (4096,4.99) (8192,9.73) (16384,19.38)
(32768,38.77) (65536,77.48) (131072,155.06)
};
\addplot+[red, thick, mark=star] coordinates {
(128,1.46) (256,1.45) (512,1.71) (1024,2.17)
(2048,3.26) (4096,5.47) (8192,9.78) (16384,19.02)
(32768,37.69) (65536,75.05) (131072,149.63)
};

\nextgroupplot[
    ylabel={Peak Memory (MB)},
    xticklabels={},
]

\addplot+[black, thick, mark=*] coordinates {
(128,11.9) (256,15.7) (512,29.4) (1024,81.1)
(2048,282.1) (4096,1074.1) (8192,4218.1) (16384,16746.1)
};
\addplot+[black, dashed, thick, mark=square*] coordinates {
(128,13.4) (256,21.7) (512,53.4) (1024,177.2)
(2048,666.3) (4096,2610.4) (8192,10362.6) (16384,41323.1)
};
\addplot+[blue, mark=triangle*] coordinates {
(128,22.9) (256,35.6) (512,61.1) (1024,112.1)
(2048,214.1) (4096,418.1) (8192,826.1) (16384,1642.1)
(32768,3274.1) (65536,6538.1) (131072,13066.1)
};
\addplot+[orange, mark=triangle*] coordinates {
(128,31.5) (256,52.9) (512,95.6) (1024,181.1)
(2048,352.2) (4096,694.2) (8192,1378.2) (16384,2746.3)
(32768,5482.5) (65536,10954.9) (131072,21899.6)
};
\addplot+[purple, mark=x] coordinates {
(128,27.8) (256,45.4) (512,80.7) (1024,151.2)
(2048,292.3) (4096,574.3) (8192,1138.4) (16384,2266.7)
(32768,4523.2) (65536,9036.2) (131072,18062.2)
};
\addplot+[red, thick, mark=star] coordinates {
(128,36.2) (256,62.2) (512,114.2) (1024,151.7)
(2048,160.7) (4096,178.7) (8192,214.7) (16384,314.2)
(32768,618.2) (65536,1226.2) (131072,2442.2)
};

\nextgroupplot[
    ylabel={Throughput (tok/s)},
    xlabel={Sequence length},
]

\addplot+[black, thick, mark=*] coordinates {
(128,309605) (256,601515) (512,1140190) (1024,1535463)
(2048,984540) (4096,581929) (8192,358680) (16384,191393)
};
\addplot+[black, dashed, thick, mark=square*] coordinates {
(128,223026) (256,457217) (512,812948) (1024,1132992)
(2048,733302) (4096,418302) (8192,219174) (16384,115845)
};
\addplot+[blue, mark=triangle*] coordinates {
(128,220935) (256,441509) (512,715210) (1024,869111)
(2048,973777) (4096,1046247) (8192,973413) (16384,870581)
(32768,893397) (65536,894465) (131072,887852)
};
\addplot+[orange, mark=triangle*] coordinates {
(128,132910) (256,260132) (512,396415) (1024,466298)
(2048,511045) (4096,544607) (8192,552194) (16384,463905)
(32768,453914) (65536,453236) (131072,450163)
};
\addplot+[purple, mark=x] coordinates {
(128,202368) (256,400720) (512,514787) (1024,660425)
(2048,762536) (4096,820486) (8192,841800) (16384,845487)
(32768,845161) (65536,845833) (131072,845286)
};
\addplot+[red, thick, mark=star] coordinates {
(128,87397) (256,176813) (512,299006) (1024,472417)
(2048,628732) (4096,748980) (8192,837345) (16384,861529)
(32768,869331) (65536,873279) (131072,875953)
};

\end{groupplot}
\end{tikzpicture}
\end{figure}
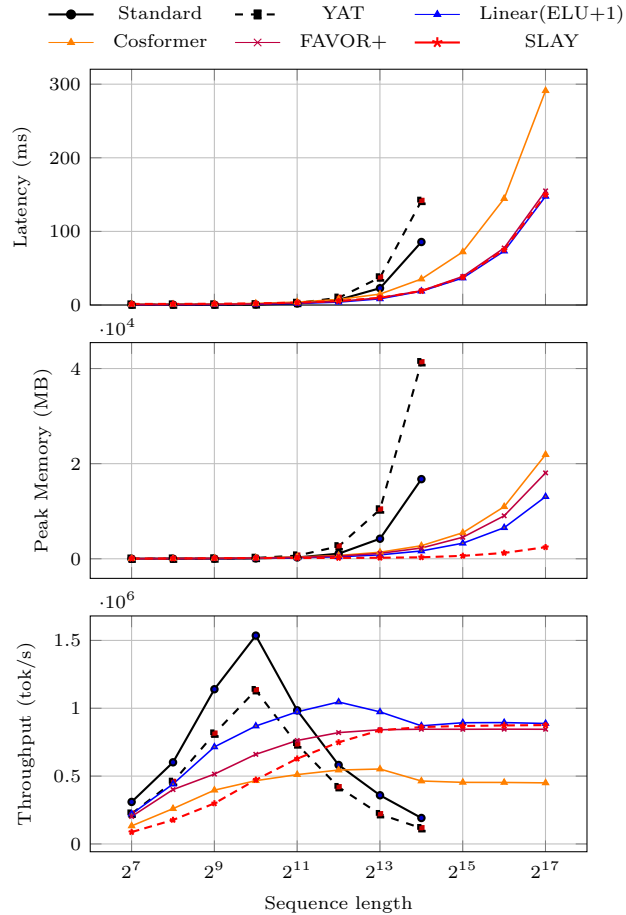

\paragraph{Results overview.}
Figure~\ref{fig:scaling} summarizes the scaling behavior. Quadratic attention mechanisms (standard softmax and exact YAT) exhibit rapidly increasing latency and memory usage, failing beyond \textbf{16K} tokens. In contrast, linear attention methods scale approximately linearly and remain stable up to \textbf{128K} tokens.

SLAY's scaling behavior closely follows and in some cases outperforms other linear attention mechanisms while substantially reducing memory usage relative to its quadratic counterparts. At long sequence lengths, SLAY uses orders of magnitude less memory than exact methods, enabling sequences well beyond where quadratic methods encounter out-of-memory failures. SLAY sustains high throughput comparable to other linear baselines, demonstrating that its added geometric structure does not compromise scalability (see Appendix~\ref{app:scaling-analysis} for detailed comparisons).

\subsection{SLAY on Synthetic Tasks}
\label{subsec:synthetic-experiments}

We evaluate all attention mechanisms on a suite of synthetic sequence modeling tasks
covering basic operations, arithmetic, long-range dependencies, memory, and reasoning. The full results and descriptions of these tasks can be found in the Appendix.

\begin{table}[htbp]
\centering
\small
\caption{Average accuracy by task category. The name of different methods is the same as in Fig. \ref{fig:scaling}.}
\label{tab:summary}

\textbf{(a) Core capabilities}

\makebox[\columnwidth][c]{%
\setlength{\tabcolsep}{3pt}%
\begin{tabular}{lcccc}
\toprule
Method & Basic & Arithmetic & Long-Range & Memory \\
\midrule
Standard        & 0.60 & 0.54 & 0.68 & 0.73 \\
YAT  & 0.52 & 0.53 & 0.68 & 0.73 \\
FAVOR+       & 0.45 & 0.57 & 0.68 & 0.73 \\
Linear          & 0.44 & 0.60 & 0.67 & 0.72 \\
\textbf{SLAY}   & \textbf{0.57} & \textbf{0.57} & \textbf{0.68} & \textbf{0.73} \\
\bottomrule
\end{tabular}%
}

\vspace{0.8em}

\textbf{(b) Higher-level behavior}

\makebox[\columnwidth][c]{%
\begin{tabular}{lccc}
\toprule
Method & Patterns & Reasoning & Robustness \\
\midrule
Standard        & 0.83 & 0.56 & 0.80 \\
YAT  & 0.81 & 0.56 & 0.80 \\
FAVOR+       & 0.85 & 0.56 & 0.80 \\
Linear          & 0.85 & 0.56 & 0.80 \\
\textbf{SLAY}   & \textbf{0.86} & \textbf{0.57} & \textbf{0.80} \\
\bottomrule
\end{tabular}%
}

\end{table}

Overall, SLAY matches or exceeds the performance of existing linearized attention mechanisms across the full spectrum of evaluated tasks. As shown in Table~\ref{tab:summary}, SLAY consistently performs on par with the strongest linear baselines on core capabilities such as long‑range dependency modeling and memory, while closing much of the gap to standard attention on basic operations. Notably, SLAY demonstrates competitive or superior accuracy in arithmetic and pattern‑based tasks, and achieves the highest average performance on higher‑level behaviors, including patterns and reasoning. These results indicate that SLAY preserves the representational strengths of linear attention while, in many settings, providing measurable gains over comparable linearized approaches.

\subsection{Extreme Classification Evaluation}
\label{subsec:extreme-classification}

We further evaluate the proposed approach in an extreme classification setting using the \textsc{Eurlex} dataset. Specifically, we compare \textbf{SLAY} against the Performer FAVOR+ mechanism.

\begin{table}[htbp]
\centering
\caption{Extreme Classification Benchmark Results on Eurlex-4K}
\label{tab:extreme_results}
\begin{tabular}{lcc}
\toprule
Metric & SLAY (Approx) & Performer \\
\midrule
P@1 & \textbf{0.4978} & 0.3442 \\
P@3 & \textbf{0.3953} & 0.2703 \\
P@5 & \textbf{0.3261} & 0.2263 \\
PSP@1 & \textbf{0.9391} & 0.5970 \\
PSP@3 & \textbf{0.7862} & 0.4785 \\
PSP@5 & \textbf{0.6693} & 0.4141 \\
\bottomrule
\end{tabular}
\end{table}

As shown in Table~\ref{tab:extreme_results}, \textbf{SLAY} achieves higher scores across the board versus the Performer. These results suggest that SLAY provides a favorable trade-off between modeling capacity and efficiency in large-label regimes, making it well-suited for extreme classification tasks.

\subsection{SLAYformer: SLAY in Transformers}
\label{subsec:slayformer}

We now evaluate \textbf{SLAY} in the context of fully trained transformer language models, focusing on how our
\emph{linear‑time} attention mechanism compares to both standard softmax attention and other linear-time approximations such as Performers and Cosformers.
To this end, we instantiate a \textbf{SLAYformer}, i.e., a GPT‑2 Small–scale model equipped with Spherical
Linearized YAT attention, and compare it against both quadratic softmax attention and state‑of‑the‑art
linear attention baselines under strictly controlled training conditions.

All models are trained to the same token budget in accordance with the Chinchilla scaling law, and share
identical architecture, optimization hyperparameters, and training data. This setup isolates the effect
of the attention mechanism itself.
\begin{table}[htbp]
\centering
\small
\caption{\small{
Validation performance comparison of attention mechanisms after satisfying the Chinchilla scaling law. Full model details available in Appendix \ref{app:experiments}.
}}
\label{tab:attention-val-comparison}
\resizebox{1\columnwidth}{!}{%
\begin{tabular}{llcc}
\toprule
Method & Complexity & Val Loss $\downarrow$ & PPL $\downarrow$ \\
\midrule
\multicolumn{4}{l}{\textit{Quadratic Attention}} \\
\quad \textbf{Yat (Exact)} & $O(n^2)$ & \textbf{4.5747} & \textbf{97.03} \\
\quad Standard Softmax & $O(n^2)$ & 4.6417 & 103.73 \\
\quad Yat (Spherical) & $O(n^2)$ & 4.7180 & 112.00 \\
\midrule
\multicolumn{4}{l}{\textit{Linear Attention}} \\
\quad \textbf{SLAYformer (Ours)} & $O(n)$ & \textbf{4.6760} & \textbf{107.35} \\
\quad Linear (ELU+1) & $O(n)$ & 5.0884 & 161.99 \\
\quad Cosformer & $O(n)$ & 5.1983 & 180.97 \\
\quad FAVOR+ (Performer) & $O(n)$ & 5.4524 & 233.32 \\
\bottomrule
\end{tabular}%
}
\end{table}

Table~\ref{tab:attention-val-comparison} reports validation loss and perplexity at convergence.
Strikingly, the SLAYformer achieves performance \emph{within a narrow margin of standard softmax attention},
despite operating with linear time and memory complexity.
In contrast, existing linear‑time attention mechanisms such as FAVOR+ (Performer), Linear (ELU+1) and Cosformer exhibit a
substantial degradation in both loss and perplexity relative to standard softmax attention. Notably, the quadratic versions of SLAY, namely Exact YAT and Spherical YAT offer very competitive performances, even surpassing Standard Softmax in the case of Exact Yat. 

To our knowledge, this is the first demonstration of a linear‑time attention mechanism that approaches
standard softmax attention this closely in a fully trained transformer language model.

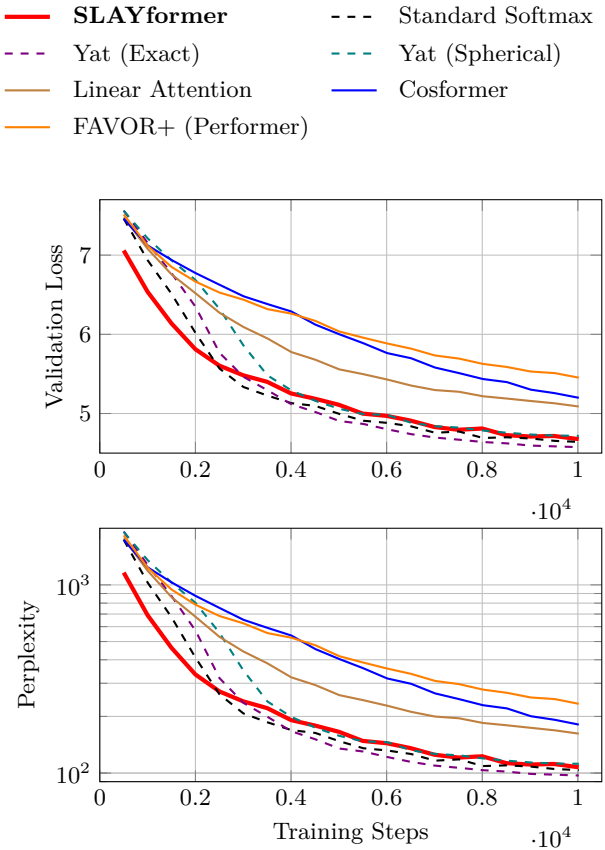
\begin{figure}[h]
\centering
\small
\captionsetup{width=0.95\columnwidth}
\caption{\small{
Validation loss (top) and validation perplexity (bottom) as a function of training steps after satisfying the Chinchilla scaling law with 125M parameters and 2.5B tokens.
}}
\label{fig:val-curves-group}
\vspace{0.5em}

\begin{tikzpicture}
\begin{groupplot}[
    group style={
        group size=1 by 2,
        horizontal sep=1.5cm
    },
    width=1\columnwidth,
    height=0.6\columnwidth,
    xlabel={Training Steps},
    grid=both,
    tick label style={font=\small},
    label style={font=\small},
    legend style={
        at={(0.4, 1.20)},
        anchor=south,
        legend columns=2,
        draw=none,
        /tikz/column sep=6pt,
        row sep=2pt
    },
    legend cell align=left
]

\nextgroupplot[
    ylabel={Validation Loss},
    xlabel={},
    xmin=0,
    xmax=10500,
    ymin=4.5,
    ymax=7.7,
]

\addplot[ultra thick, red, solid] table[x=Step, y=Loss] {data/slay_loss.dat};
\addlegendentry{\textbf{SLAYformer}}

\addplot[thick, dashed, black] table[x=Step, y=Loss] {data/softmax_loss.dat};
\addlegendentry{Standard Softmax}

\addplot[thick, dashed, violet] table[x=Step, y=Loss] {data/yat_exact_loss.dat};
\addlegendentry{Yat (Exact)}

\addplot[thick, dashed, teal] table[x=Step, y=Loss] {data/yat_spherical_loss.dat};
\addlegendentry{Yat (Spherical)}

\addplot[thick, solid, brown] table[x=Step, y=Loss] {data/linear_loss.dat};
\addlegendentry{Linear Attention}

\addplot[thick, solid, blue] table[x=Step, y=Loss] {data/cosformer_loss.dat};
\addlegendentry{Cosformer}

\addplot[thick, solid, orange] table[x=Step, y=Loss] {data/performer_loss.dat};
\addlegendentry{FAVOR+ (Performer)}

\nextgroupplot[
    ylabel={Perplexity},
    xmin=0,
    xmax=10500,
    ymin=90,
    ymax=2000,
    ymode=log,
]

\addplot[ultra thick, red, solid] table[x=Step, y=PPL] {data/slay_ppl.dat};

\addplot[thick, dashed, black] table[x=Step, y=PPL] {data/softmax_ppl.dat};

\addplot[thick, dashed, violet] table[x=Step, y=PPL] {data/yat_exact_ppl.dat};

\addplot[thick, dashed, teal] table[x=Step, y=PPL] {data/yat_spherical_ppl.dat};

\addplot[thick, solid, brown] table[x=Step, y=PPL] {data/linear_ppl.dat};

\addplot[thick, solid, blue] table[x=Step, y=PPL] {data/cosformer_ppl.dat};

\addplot[thick, solid, orange] table[x=Step, y=PPL] {data/performer_ppl.dat};

\end{groupplot}
\end{tikzpicture}
\end{figure}

Figure~\ref{fig:val-curves-group} provide further insight into the training dynamics.
Across the entire optimization trajectory, the SLAYformer exhibits stable training behavior and converges
at a rate comparable to softmax and yat attention.
Importantly, the gap between SLAY and softmax remains small throughout training, while the gap between SLAY
and other linear attention methods remains consistently large.
This indicates that SLAY’s advantage is not limited to early optimization or transient effects, but persists
through full convergence.

Taken together, these results suggest that \textbf{SLAY represents a qualitative advance over prior linear‑time
attention mechanisms.}
Rather than trading expressivity for efficiency, SLAYformers retain modeling performance close to that of
quadratic softmax attention while enabling scalable training and inference at long context lengths.
This positions SLAY as a strong candidate for a new state of the art among linear‑time attention transformers.

This empirical behavior is particularly notable in light of the theoretical properties of the YAT kernel
introduced in Section~\ref{sec:intro}.
By enabling an efficient linear implementation of YAT, SLAY makes it possible to train large‑scale
transformers that preserve the kernel’s geometric and theoretical structure, without incurring the
prohibitive quadratic cost of exact attention.
As a result, \textbf{SLAY provides a practical and principled path toward long‑context transformer models that
combine scalability and geometric integrity with near‑softmax performance.}

\section{Conclusion}
\label{sec:conclusion}

We introduced SLAY (Spherical Linearized Attention with Yat kernel), a linear‑time attention mechanism that makes the $\mathbb{E}$‑product from Neural Matter Networks practical for long‑context sequence modeling. By enforcing unit‑norm constraints and using a strictly positive random‑feature approximation derived via Bernstein’s theorem, SLAY yields a bounded and factorizable attention kernel compatible with existing linear attention frameworks.

We show theoretically that SLAY preserves key structural properties of the $\mathbb{E}$‑product, including self‑regulation and superposition. Empirically, SLAYformers train stably, scale linearly in time and memory, and process sequences up to \textbf{30$\times$ longer} than standard softmax attention.

Across all evaluated settings, SLAYformers achieve performance that is very close to its quadratic but precise counterparts: standard softmax attention, exact yat and spherical yat, all the while substantially outperforming prior linear‑time attention mechanisms such as Performers, Cosformers and Linear Attention. To the best of our knowledge, \textbf{SLAY is the first linear‑time attention mechanism to consistently approach softmax‑level performance without sacrificing scalability}, representing a significant step toward practical and efficient long‑context Transformers.

\section*{Impact Statement}

This work advances efficient Transformer architectures by introducing a linear‑time attention mechanism that closely matches the performance of standard softmax attention. By narrowing the long‑standing gap between quadratic attention and scalable linear alternatives, SLAYformers reduce the need to trade modeling quality for computational efficiency.

As a result, SLAY has the potential to lower the computational and energy costs of training and deploying long‑context models, supporting more sustainable and accessible large‑scale language modeling. More broadly, this work also highlights the promise of geometry‑aware kernel design for efficient attention, and we support reproducibility through detailed experimental settings and a publicly available code repository.
\bibliography{references}
\bibliographystyle{icml2026}


\clearpage
\onecolumn

\section*{Appendix}
\appendix
\section{Background Results}
\label{app:background}

This section collects standard background facts used in the main text.

\begin{adjustwidth}{1.5em}{1.5em}
\begin{lemma}[Bernstein Representation Applicability]
\label{lem:complete-monotone}
For $\epsilon > 0$ and $C = 2 + \epsilon$, the variable $y = C - 2x$ satisfies $y \geq \epsilon > 0$ for all $x \in [-1,1]$. Hence Bernstein's representation $1/(C - 2x) = \int_0^\infty e^{-s(C-2x)}\,ds$ applies throughout the domain.
\end{lemma}
\begin{proof}
See Appendix~\ref{app:additional-proofs}.
\end{proof}
\end{adjustwidth}

\section{Geometric Interpretation and Invariances}
\label{app:geom}

This appendix expands on the geometric view of the spherical \E-product discussed in Section~\ref{sec:method}. On $\Sph^{d-1}$, the squared chordal distance satisfies
$d_{\Sph^{d-1}}(\widehat{\mathbf{q}},\widehat{\mathbf{k}})^2
= 2\bigl(1-\widehat{\mathbf{q}}^\top \widehat{\mathbf{k}}\bigr)$,
so the spherical kernel can be interpreted as an $\epsilon$-regularized chordal-distance interaction.

\begin{adjustwidth}{1.5em}{1.5em}
\begin{proposition}[Geometric Origin]
\label{prop:geometric}
The spherical \E-product is an $\epsilon$-regularized chordal-distance interaction on $\Sph^{d-1}$:
\[
\E_{\mathrm{sph}}(\widehat{\mathbf{q}}, \widehat{\mathbf{k}})
=
\frac{\langle \widehat{\mathbf{q}}, \widehat{\mathbf{k}} \rangle^2}
     {d_{\Sph^{d-1}}(\widehat{\mathbf{q}}, \widehat{\mathbf{k}})^2 + \epsilon},
\]
where the numerator captures directional alignment and the denominator enforces locality via chordal proximity.
\end{proposition}
\end{adjustwidth}

Since the kernel depends only on
$\widehat{\mathbf{q}}^\top \widehat{\mathbf{k}}$,
it belongs to the class of isotropic kernels on the sphere.
All spherical positive-definiteness claims in the main text assume $d\ge 2$,
where Schoenberg's characterization applies \citep{schoenberg1942}.

\begin{adjustwidth}{1.5em}{1.5em}
\begin{remark}[Geometric Invariances]
\label{rem:invariances}
The spherical $\E$-product is invariant under
(i) rotations:
$\E_{\mathrm{sph}}(\mathbf{R}\widehat{\mathbf{q}}, \mathbf{R}\widehat{\mathbf{k}})
= \E_{\mathrm{sph}}(\widehat{\mathbf{q}}, \widehat{\mathbf{k}})$
for all $\mathbf{R} \in SO(d)$,
and (ii) uniform scaling prior to normalization.
Note that while
$(\widehat{\mathbf{q}}^\top \widehat{\mathbf{k}})^2$
is even under sign flips, the full kernel
\[
\E_{\mathrm{sph}}(\widehat{\mathbf{q}},\widehat{\mathbf{k}})
=
\frac{(\widehat{\mathbf{q}}^\top \widehat{\mathbf{k}})^2}
     {(2+\epsilon)-2\widehat{\mathbf{q}}^\top\widehat{\mathbf{k}}}
\]
is not invariant under $\widehat{\mathbf{q}} \mapsto -\widehat{\mathbf{q}}$
in general.
\end{remark}
\end{adjustwidth}

\begin{adjustwidth}{1.5em}{1.5em}
\begin{proposition}[PRF Unbiasedness]
\label{prop:prf-unbiased}
For $\widehat{\mathbf{q}}, \widehat{\mathbf{k}} \in \Sph^{d-1}$ and
$\{\boldsymbol{\omega}_i\}_{i=1}^D
\stackrel{\mathrm{iid}}{\sim} \mathcal{N}(\mathbf{0}, \mathbf{I}_d)$:
\[
\mathbb{E}\!\left[
\left\langle
\phi_{\text{PRF}}(\widehat{\mathbf{q}}; s),
\phi_{\text{PRF}}(\widehat{\mathbf{k}}; s)
\right\rangle
\right]
=
e^{2s\,\widehat{\mathbf{q}}^\top \widehat{\mathbf{k}}}.
\]
\end{proposition}
\begin{proof}
See Appendix~\ref{app:additional-proofs}.
\end{proof}
\end{adjustwidth}

\begin{adjustwidth}{1.5em}{1.5em}
\begin{lemma}[Positive Mixture Closure]
\label{lem:mixture-pd}
If $\{k_s\}_{s \geq 0}$ is a family of positive-definite kernels on $\mathcal{X}$
and $w(s) \geq 0$ is a nonnegative measure, then
$k(x,y) = \int_0^\infty w(s)\,k_s(x,y)\,ds$
is PD on $\mathcal{X}$ (a standard closure property of PD kernels;
see, e.g., \citep{aronszajn1950,scholkopf2002learning}).
\end{lemma}
\end{adjustwidth}

\begin{adjustwidth}{1.5em}{1.5em}
\begin{theorem}[Tensor Kernel Decomposition]
\label{thm:tensor-rkhs}
Let $k_1, k_2$ be positive-definite kernels on $\mathcal{X}$
with RKHSs $\mathcal{H}_1, \mathcal{H}_2$.
Then the product kernel
$k(x,y) = k_1(x,y) \cdot k_2(x,y)$
is positive definite with RKHS
$\mathcal{H} = \mathcal{H}_1 \otimes \mathcal{H}_2$
(see, e.g., \citep{aronszajn1950,berlinet2004}).
\end{theorem}
\end{adjustwidth}

\section{Preliminaries: Polynomial Kernel Approximations}
\label{app:prelim-poly}

We summarize three approximations of the degree-2 polynomial kernel
$k_{\text{poly}}(\mathbf{x},\mathbf{y})=(\mathbf{x}^\top \mathbf{y})^2$
used in our implementation and ablations.
Let $\mathbf{x},\mathbf{y}\in\mathbb{R}^d$.

\paragraph{Random Maclaurin (RM).}
Draw Rademacher vectors
$\mathbf{r}_i, \mathbf{s}_i\in\{\pm1\}^d$ and define
\[
\phi_{\text{RM}}(\mathbf{x})
=
\frac{1}{\sqrt{P}}
\bigl[(\mathbf{r}_i^\top \mathbf{x})(\mathbf{s}_i^\top \mathbf{x})\bigr]_{i=1}^P.
\]
Then
$\mathbb{E}[\langle \phi_{\text{RM}}(\mathbf{x}),
\phi_{\text{RM}}(\mathbf{y})\rangle]
=
(\mathbf{x}^\top \mathbf{y})^2$
\citep{karkarnick2012}.
RM is unbiased but can have high variance for small $P$.

\paragraph{Nystrom features.}
Let anchors
$A=\{\mathbf{a}_1,\dots,\mathbf{a}_P\}\subset\mathbb{R}^d$ and define
$K_{AA}\in\mathbb{R}^{P\times P}$ with
$(K_{AA})_{ij}=(\mathbf{a}_i^\top \mathbf{a}_j)^2$.
Given $K_{AA}+\lambda I$, define
\[
\phi_{\text{Nys}}(\mathbf{x})
=
(K_{\mathbf{x}A})(K_{AA}+\lambda I)^{-1/2},
\quad
K_{\mathbf{x}A}
=
[(\mathbf{x}^\top \mathbf{a}_i)^2]_{i=1}^P.
\]
This yields a low-rank approximation whose quality depends on anchor
coverage and conditioning \citep{williams2001nystrom}.

\paragraph{Anchor (low-rank) features.}
Using the same anchors $A$, define
\[
\phi_{\text{Anc}}(\mathbf{x})
=
\frac{1}{\sqrt{P}}
\bigl[(\mathbf{x}^\top \mathbf{a}_i)^2\bigr]_{i=1}^P.
\]
This is a simple low-rank approximation; it is not unbiased in general
but is often stable for small $P$.

\paragraph{Comparison.}
RM is unbiased but variance-dominated at small feature budgets;
Nystrom reduces variance if anchors are well-conditioned;
anchor features are computationally simplest and empirically most stable
at small $P$.

\section{Ablation: Polynomial Approximation Sweep}
\label{app:poly-ablation}

This appendix reports the multi-scale kernel-fidelity sweep referenced in Section~\ref{subsec:poly-ablation}. All variants tie the $\mathrm{QKV}$ and output projection weights and compare outputs against exact kernel-normalized spherical \E-attention.

\begin{table}[t]
\centering
\small
\caption{Multi-scale ablation over feature budgets for polynomial-kernel approximations. We compare attention outputs against \emph{exact kernel-normalized} spherical YAT with tied QKV/out projections. Lower Rel.~$\ell_2$ is better; latency is forward-pass time.}
\label{tab:poly-sweep}
\begin{tabular}{llcccccc}
\toprule
Scale & Method & $T$ & $R$ & $M$ & $P$ & Rel.~$\ell_2\downarrow$ & Latency (ms)$\downarrow$\\
\midrule
Small & Exact (Spherical) & 128 & 2 & 8 & 8 & 0.0000 & 3.12\\
 & Laplace-only &  &  &  &  & 0.5870 & 2.78\\
 & Anchor &  &  &  &  & 0.6626 & 3.82\\
 & Hadamard (shared $\omega$) &  &  &  &  & 0.8237 & 3.34\\
 & Nystrom &  &  &  &  & 22.9072 & 3.41\\
 & TensorSketch &  &  &  &  & 474075.1562 & 5.17\\
 & Random Maclaurin &  &  &  &  & 2195912.7500 & 5.59\\
\addlinespace
Medium & Exact (Spherical) & 256 & 2 & 16 & 16 & 0.0000 & 0.79\\
 & Laplace-only &  &  &  &  & 0.5417 & 16.10\\
 & Anchor &  &  &  &  & 0.5667 & 18.54\\
 & Hadamard (shared $\omega$) &  &  &  &  & 0.6609 & 17.44\\
 & Nystrom &  &  &  &  & 61.6529 & 18.46\\
 & TensorSketch &  &  &  &  & 214115.9844 & 19.01\\
 & Random Maclaurin &  &  &  &  & 1715766.8750 & 18.80\\
\addlinespace
Large & Exact (Spherical) & 512 & 2 & 32 & 32 & 0.0000 & 5.02\\
 & Laplace-only &  &  &  &  & 0.4850 & 1905.80\\
 & Anchor &  &  &  &  & 0.4939 & 489.42\\
 & Hadamard (shared $\omega$) &  &  &  &  & 0.6793 & 1932.07\\
 & Nystrom &  &  &  &  & 28.1970 & 569.64\\
 & TensorSketch &  &  &  &  & 461739.0312 & 547.76\\
 & Random Maclaurin &  &  &  &  & 1772757.5000 & 551.43\\
\addlinespace
\bottomrule
\end{tabular}
\end{table}

\section{Integral Representation of the Spherical \E-Product}
\label{app:integral}

In this appendix, we provide a detailed derivation of the integral representation used to linearize the spherical \E-product kernel.

Recall that under the unit-norm constraint
$\widehat{\mathbf{q}}, \widehat{\mathbf{k}} \in \Sph^{d-1}$,
the \E-product reduces to
\begin{equation}
\E_{\mathrm{sph}}(\widehat{\mathbf{q}}, \widehat{\mathbf{k}})
=
\frac{x^2}{C - 2x},
\qquad
x = \widehat{\mathbf{q}}^\top \widehat{\mathbf{k}} \in [-1,1],
\quad
C = 2 + \epsilon.
\label{eq:app-spherical}
\end{equation}

The function $g(y) = 1/y$ is completely monotone on $(0,\infty)$ and therefore admits the Bernstein representation
\begin{equation}
\frac{1}{y} = \int_0^\infty e^{-s y}\, ds.
\label{eq:app-bernstein}
\end{equation}

Applying this identity with $y = C - 2x$, we obtain
\begin{align}
\E_{\mathrm{sph}}(\widehat{\mathbf{q}}, \widehat{\mathbf{k}})
&= x^2 \int_0^\infty e^{-s(C - 2x)} \, ds \\
&= \int_0^\infty e^{-sC} \, x^2 e^{2s x} \, ds.
\label{eq:app-integral}
\end{align}

This representation expresses the spherical \E-product as a positively weighted mixture of a degree-2 polynomial kernel and an exponential kernel in the angular variable $x$. This decomposition forms the basis for the random-feature approximation introduced in the main text.

\section{Random Feature Construction and Unbiasedness}
\label{app:features}

In this appendix, we provide additional details on the random-feature construction used to approximate the integrand appearing in the spherical \E-product representation and justify its unbiasedness.

Recall from Eq.~\eqref{eq:integral} that the spherical \E-product admits the decomposition
\[
\E_{\mathrm{sph}}(\widehat{\mathbf{q}}, \widehat{\mathbf{k}})
=
\int_0^\infty e^{-sC} \, x^2 e^{2s x} \, ds,
\qquad
x = \widehat{\mathbf{q}}^\top \widehat{\mathbf{k}}.
\]

\paragraph{Polynomial component.}
The term
$x^2 = (\widehat{\mathbf{q}}^\top \widehat{\mathbf{k}})^2$
corresponds to a homogeneous degree-2 polynomial kernel.
This kernel admits an explicit feature map given by
\[
\phi_{\mathrm{poly}}(\mathbf{u})
=
\mathrm{vec}(\mathbf{u}\mathbf{u}^\top),
\]
or an approximate variant implemented via tensor sketching.
In both cases, the inner product of feature maps yields an unbiased
estimator of the polynomial kernel.

\paragraph{Exponential component.}
The exponential term $e^{2s x}$ is approximated using strictly positive
random features.
For random projections $\boldsymbol{\omega}$ drawn from a Gaussian or
orthogonal distribution, the feature map
\[
\phi_{\mathrm{PRF}}(\mathbf{u}; s)
=
\frac{1}{\sqrt{D}}
\exp\!\left(\sqrt{2s}\,\boldsymbol{\omega}^\top \mathbf{u} - s\right)
\]
satisfies
\[
\mathbb{E}\big[
\langle
\phi_{\mathrm{PRF}}(\widehat{\mathbf{q}}; s),
\phi_{\mathrm{PRF}}(\widehat{\mathbf{k}}; s)
\rangle
\big]
=
e^{2s\, \widehat{\mathbf{q}}^\top \widehat{\mathbf{k}}}.
\]

\paragraph{Tensor product approximation.}
Since the polynomial component can be computed exactly (or approximated
in practice) and the exponential component is estimated with unbiased
PRFs, their tensor product
\[
\phi_{\mathrm{poly}}(\mathbf{u})
\otimes
\phi_{\mathrm{PRF}}(\mathbf{u}; s)
\]
is an unbiased estimator of the product kernel
$x^2 e^{2s x}$ by linearity of expectation.
Approximating the outer integral using quadrature preserves unbiasedness
up to the discretization error introduced by the numerical integration
scheme.

\paragraph{On ``pure Laplace'' forms.}
If one insists on a nonnegative Laplace mixture of plain exponentials
$e^{2sx}$, then $x^2/(C-2x)$ cannot be represented exactly, because
$k(0)=0$ whereas
\[
\int_0^\infty w(s)\,e^{2s\cdot 0}\,ds
=
\int_0^\infty w(s)\,ds
\ge 0
\]
for any $w\ge 0$.
There is, however, an exact decomposition that removes the explicit
$x^2$ factor at the cost of an affine correction term:
\[
\frac{x^2}{C-2x}
=
\frac{C^2}{4}\int_0^\infty e^{-Cs}\,e^{2sx}\,ds
\;-\;
\frac{C}{4}
\;-\;
\frac{x}{2}.
\]
This identity follows from
$x^2 e^{2sx}=\tfrac{1}{4}\,\partial_s^2 e^{2sx}$
and two integrations by parts, whose boundary terms yield the affine
correction.
While this removes the need for polynomial random features, it
introduces signed components (through the subtraction of constant and
linear kernels) and therefore does not retain the
``strictly positive feature map'' and denominator-positivity guarantees
emphasized in the main construction.

\paragraph{Hadamard (elementwise) fusion variant.}
Some implementations replace the tensor product with elementwise
(Hadamard) fusion,
\[
\phi_{\mathrm{had}}(\mathbf{u}; s)
=
\sqrt{w_r}\,
\bigl(
\phi_{\mathrm{poly}}(\mathbf{u})
\odot
\phi_{\mathrm{PRF}}(\mathbf{u}; s)
\bigr),
\]
which yields a valid positive feature map but targets a \emph{different}
kernel than the tensor product.
In particular, the expected inner product becomes
\[
\mathbb{E}\big[
\langle
\phi_{\mathrm{had}}(\widehat{\mathbf{q}}; s),
\phi_{\mathrm{had}}(\widehat{\mathbf{k}}; s)
\rangle
\big]
\;\approx\;
(\widehat{\mathbf{q}}^\top \widehat{\mathbf{k}})^2
\,e^{2s\,\widehat{\mathbf{q}}^\top \widehat{\mathbf{k}}}
\quad
\text{only if }
\phi_{\mathrm{poly}}
\text{ and }
\phi_{\mathrm{PRF}}
\text{ are aligned feature maps.}
\]
With standard independent random features, Hadamard fusion instead
corresponds to a product of marginal kernels across matched feature
indices, which generally introduces bias relative to the target
integrand kernel.
The benefit is computational: it avoids the $D_p\times D_r$ tensor
expansion and reduces memory, but at the cost of a kernel mismatch.
We therefore treat Hadamard fusion as a fast baseline rather than the
primary estimator of the spherical \E-product.

\section{Positivity and Stability Guarantees}
\label{app:stability}

This appendix provides additional justification for the positivity and numerical stability properties of the proposed linearized \E-attention mechanism.

\paragraph{Positivity.}
All components of the \emph{target} spherical kernel are non-negative. Moreover, if the polynomial feature map is computed exactly (or with a positivity-preserving approximation), then the corresponding approximate scores are non-negative:
\begin{itemize}[leftmargin=1.4em]
    \item The polynomial term
    $(\widehat{\mathbf{q}}^\top \widehat{\mathbf{k}})^2$
    is non-negative for all
    $\widehat{\mathbf{q}}, \widehat{\mathbf{k}} \in \Sph^{d-1}$.
    \item The exponential term $e^{2s x}$ is strictly positive for all
    $s \ge 0$ and $x \in [-1,1]$.
    \item The quadrature weights $w_r$ are non-negative.
\end{itemize}
Consequently, under this condition the approximate attention scores produced by the tensor-product random features are non-negative.

\paragraph{Numerical stability.}
The boundedness of the spherical \E-product on $\Sph^{d-1}$
(Proposition~\ref{prop:bounded}) implies that attention scores remain uniformly bounded.
Combined with positivity, this prevents the numerical instabilities associated with oscillatory random features and negative attention weights.
This behavior mirrors stability properties previously observed in positive random
feature–based linear attention mechanisms \citep{performers}.

\section{Experimental and Implementation Details}
\label{app:experiments}

This appendix summarizes additional experimental and implementation details to facilitate reproducibility.

\paragraph{Random feature configuration.}
Unless otherwise stated, all experiments use a fixed number of random features per attention head. Quadrature nodes and weights are chosen using standard numerical integration schemes and shared across heads and layers.

\paragraph{Normalization.}
Queries and keys are explicitly normalized to unit norm prior to feature computation. This normalization is applied per attention head and does not introduce additional learnable parameters.

\paragraph{Hardware and software.}
All experiments were conducted using PyTorch on NVIDIA A100 GPUs with ~80\,GB of memory using DeepSpeed for distributed training. Attention-only benchmarks use custom linear-attention operators, while full-model experiments rely on standard PyTorch modules augmented with the proposed attention mechanism.

\paragraph{Model Architecture}. We use a standard GPT-2 Small architecture (124M parameters) with 12 layers, 12 attention heads, and an embedding dimension of $d_{model}=768$. The MLP block is the standard GPT-2 MLP (LayerNorm + GELU).

\paragraph{Hyperparameters}. All models use the same training hyperparameters:
\begin{itemize}[leftmargin=*]
    \setlength\itemsep{0em}
    \item \textbf{Optimizer}: AdamW, LR $1 \times 10^{-4}$, Weight Decay $0.01$.
    \item \textbf{Batch Size}: 32 per GPU (Ramp-up: 8$\to$32 over 500 steps).
    \item \textbf{Dropout}: $0.1$ (embeddings and attention).
\end{itemize}

\paragraph{Training configuration.}
Optimizer settings, learning rates, batch sizes, and training schedules are kept identical across softmax and \E-based attention variants unless otherwise specified. This ensures that observed differences are attributable to the attention mechanism rather than auxiliary training effects.

\paragraph{Ablation protocol (polynomial approximations).}
The kernel-fidelity ablation in Section~\ref{subsec:poly-ablation} is implemented in \texttt{tests/ablation\_poly\_approx.py}. Running the script with \texttt{--sweep} produces a LaTeX table in \texttt{tables/poly\_ablation\_sweep.tex}. All variants tie the $\mathrm{QKV}$ and output projection weights and compare outputs against exact kernel-normalized spherical \E-attention.

\paragraph{Detailed Synthetic Test Results.}

Below is a summary of the categories of and specific synthetic tests ran. 

\begin{table}[t]
\centering
\small
\caption{\small{
Overview of benchmark task categories used in our evaluation.
Each category groups tasks designed to probe specific capabilities of sequence models.
See \texttt{docs/BENCHMARK\_TASKS.md} in the project code repository for detailed task descriptions.
}}
\label{tab:task-categories}
\resizebox{1\columnwidth}{!}{%
\begin{tabular}{lll}
\toprule
\textbf{Category} & \textbf{Tasks} & \textbf{Tests} \\
\midrule
Basic &
copy, sort, reverse &
Information routing \\

Memory &
retrieval, kv\_recall, first\_token, selective\_copy &
Sparse retrieval, associative memory \\

Long-Range &
long\_copy, distant\_match, multihop &
Long-range dependencies \\

Reasoning &
stack, induction, pattern &
State tracking, pattern matching \\

Arithmetic &
counting, parity, addition, modular &
Aggregation \\

Pattern &
bigram, majority &
Statistical patterns \\

Robustness &
noisy\_copy, compression &
Noise filtering \\

Aggregation &
histogram &
Multi-class counting \\
\bottomrule
\end{tabular}%
}
\end{table}

Below (Table \ref{tab:tasks-all}) provides a fully detailed results of said synthetic tests.

\setlength{\textfloatsep}{8pt}
\begin{table}[t]
\centering
\small
\caption{Synthetic task performance across all categories. Accuracy (mean $\pm$ std over 3 seeds).}
\label{tab:tasks-all}
\begin{tabular}{lccccc}
\toprule
Task & Standard & Spherical--YAT & Performer & Linear & SLAY \\
\midrule
\multicolumn{6}{l}{\emph{Basic}} \\
Copy            & 1.00$\pm$0.00 & 1.00$\pm$0.00 & 1.00$\pm$0.00 & 1.00$\pm$0.00 & 1.00$\pm$0.00 \\
Sort            & 0.28$\pm$0.02 & 0.24$\pm$0.01 & 0.27$\pm$0.02 & 0.26$\pm$0.02 & 0.29$\pm$0.01 \\
Reverse         & 0.51$\pm$0.00 & 0.33$\pm$0.02 & 0.09$\pm$0.01 & 0.05$\pm$0.01 & 0.42$\pm$0.04 \\
\midrule
\multicolumn{6}{l}{\emph{Arithmetic}} \\
Counting        & 0.72$\pm$0.01 & 0.78$\pm$0.04 & 0.81$\pm$0.06 & 0.83$\pm$0.05 & 0.74$\pm$0.13 \\
Parity          & 0.49$\pm$0.03 & 0.49$\pm$0.03 & 0.49$\pm$0.03 & 0.49$\pm$0.03 & 0.49$\pm$0.03 \\
Addition        & 0.78$\pm$0.03 & 0.68$\pm$0.16 & 0.84$\pm$0.02 & 0.91$\pm$0.04 & 0.86$\pm$0.05 \\
Modular         & 0.15$\pm$0.03 & 0.16$\pm$0.01 & 0.15$\pm$0.02 & 0.15$\pm$0.03 & 0.20$\pm$0.03 \\
\midrule
\multicolumn{6}{l}{\emph{Long-Range}} \\
Long Copy       & 1.00$\pm$0.00 & 1.00$\pm$0.00 & 1.00$\pm$0.00 & 1.00$\pm$0.00 & 1.00$\pm$0.00 \\
Distant Match   & 1.00$\pm$0.00 & 1.00$\pm$0.00 & 1.00$\pm$0.00 & 0.99$\pm$0.02 & 1.00$\pm$0.00 \\
Multihop        & 0.04$\pm$0.01 & 0.04$\pm$0.01 & 0.03$\pm$0.02 & 0.03$\pm$0.00 & 0.04$\pm$0.01 \\
\midrule
\multicolumn{6}{l}{\emph{Memory}} \\
Retrieval       & 1.00$\pm$0.00 & 1.00$\pm$0.00 & 1.00$\pm$0.00 & 1.00$\pm$0.00 & 1.00$\pm$0.00 \\
Kv Recall       & 0.02$\pm$0.02 & 0.02$\pm$0.03 & 0.03$\pm$0.00 & 0.02$\pm$0.02 & 0.02$\pm$0.01 \\
First Token     & 1.00$\pm$0.00 & 1.00$\pm$0.00 & 1.00$\pm$0.00 & 0.97$\pm$0.02 & 1.00$\pm$0.00 \\
Selective Copy  & 0.88$\pm$0.00 & 0.88$\pm$0.00 & 0.88$\pm$0.00 & 0.88$\pm$0.00 & 0.88$\pm$0.00 \\
\midrule
\multicolumn{6}{l}{\emph{Patterns}} \\
Bigram          & -- & -- & -- & -- & -- \\
Majority        & 0.78$\pm$0.07 & 0.75$\pm$0.06 & 0.82$\pm$0.02 & 0.82$\pm$0.06 & 0.84$\pm$0.03 \\
Histogram       & 0.87$\pm$0.00 & 0.87$\pm$0.00 & 0.87$\pm$0.00 & 0.87$\pm$0.00 & 0.87$\pm$0.00 \\
\midrule
\multicolumn{6}{l}{\emph{Reasoning}} \\
Stack           & 0.75$\pm$0.01 & 0.75$\pm$0.01 & 0.76$\pm$0.01 & 0.75$\pm$0.01 & 0.76$\pm$0.01 \\
Induction       & 0.02$\pm$0.02 & 0.02$\pm$0.02 & 0.02$\pm$0.01 & 0.01$\pm$0.01 & 0.03$\pm$0.01 \\
Pattern         & 0.91$\pm$0.00 & 0.91$\pm$0.00 & 0.91$\pm$0.00 & 0.91$\pm$0.00 & 0.91$\pm$0.00 \\
\midrule
\multicolumn{6}{l}{\emph{Robustness}} \\
Noisy Copy      & 1.00$\pm$0.00 & 1.00$\pm$0.00 & 1.00$\pm$0.00 & 1.00$\pm$0.00 & 1.00$\pm$0.00 \\
Compression     & 0.59$\pm$0.00 & 0.59$\pm$0.00 & 0.59$\pm$0.01 & 0.59$\pm$0.01 & 0.59$\pm$0.01 \\
\bottomrule
\end{tabular}
\end{table}

\paragraph{Code availability.}
An open-source implementation of the SLAY mechanism, including training scripts and experimental configurations, is available at
\url{https://anonymous.4open.science/r/slay-3B7C}.

\section{Implementation Notes: Quadrature Scaling and Shapes}
\label{app:impl}

\paragraph{Practical knobs and defaults.}
In our implementation, $R$ controls the quadrature accuracy of the Laplace integral and $D$ controls the Monte Carlo variance of PRF. The polynomial approximation uses either a feature dimension $D_p$ (e.g., Random Maclaurin or TensorSketch) or an anchor count $P$ (anchor features or Nystrom); by default we use anchor features because they preserve non-negativity of the polynomial component and are stable at small budgets.
The tensor-product fusion uses a sketch dimension $D_t$, trading accuracy for compute/memory. We use small stabilizers $\epsilon$ (kernel) and $\delta$ (attention denominator) for numerical robustness.

\paragraph{Remark (explicit tensor product).}
Without sketching, the per-node tensor-product feature dimension is $D_pD$ and the resulting attention cost would scale as $O(L\,R\,D_pD)$ rather than $O(L\,R\,D_t)$. We use sketching to avoid explicitly materializing Kronecker vectors while preserving the same target product-kernel structure up to controlled approximation error.

\paragraph{Causal vs. non-causal.}
The linearization applies to both causal and non-causal attention. In experiments we use a causal prefix-sum implementation for autoregressive models; for non-causal settings the same features can be used with the standard linear-attention reordering.

\section{Mathematical Tools Used (High Level)}
\label{app:tools}

Our linearization relies on
(i) the Laplace/Bernstein representation of $1/y$ for completely monotone functions
\citep{widder1941laplace,schilling2012},
(ii) closure properties of positive-definite (PD) kernels under products and nonnegative mixtures
\citep{aronszajn1950,scholkopf2002learning},
(iii) the Gaussian moment generating function to obtain unbiased positive random features for exponential dot-product kernels
\citep{performers},
and (iv) Gauss--Laguerre quadrature to discretize
$\int_0^\infty e^{-t}f(t)\,dt$
\citep{davis1984,gautschi2004}.
Numerical stabilization follows standard practice \citep{higham2002}.

\paragraph{Gauss--Laguerre scaling.}
Let $\{t_r,\alpha_r\}_{r=1}^R$ be the Gauss--Laguerre nodes and weights for
$\int_0^\infty e^{-t}f(t)\,dt$.
With $t=Cs$ and $C=2+\epsilon$, we use
\[
s_r=\frac{t_r}{C},
\qquad
w_r=\frac{\alpha_r}{C},
\]
so that
$\int_0^\infty e^{-Cs}h(s)\,ds
\approx
\sum_{r=1}^R w_r h(s_r)$.

\paragraph{Linear-attention shapes.}
If
$\Psi(\mathbf{Q}),\Psi(\mathbf{K})\in\mathbb{R}^{L\times m}$
and
$\mathbf{V}\in\mathbb{R}^{L\times d_v}$,
then
\[
\mathbf{N}
=
\Psi(\mathbf{Q})
\bigl(\Psi(\mathbf{K})^\top \mathbf{V}\bigr)
\in
\mathbb{R}^{L\times d_v},
\qquad
\mathbf{d}
=
\Psi(\mathbf{Q})
\bigl(\Psi(\mathbf{K})^\top \mathbf{1}\bigr)
\in
\mathbb{R}^{L\times 1},
\]
and we compute
$\widehat{\mathbf{Y}}_i = \mathbf{N}_i/(\mathbf{d}_i+\delta)$
with row-wise broadcasting across $d_v$.

\section{Additional Proofs}
\label{app:additional-proofs}

\begin{adjustwidth}{1.5em}{1.5em}
\begin{proposition}[Boundedness on the Unit Sphere]
\label{prop:bounded}
Let $\widehat{\mathbf{q}}, \widehat{\mathbf{k}} \in \Sph^{d-1}$.
Then the spherical \E-product satisfies
\[
0
\;\le\;
\E_{\mathrm{sph}}(\widehat{\mathbf{q}},\widehat{\mathbf{k}})
\;\le\;
\frac{1}{\epsilon}.
\]
\end{proposition}
\end{adjustwidth}

\begin{adjustwidth}{1.5em}{1.5em}
\begin{proposition}[Gradient Stability]
\label{prop:lipschitz}
There exists a constant $C_\epsilon$ such that for all
$\widehat{\mathbf{q}}, \widehat{\mathbf{k}} \in \Sph^{d-1}$:
\[
\bigl\|
\nabla_{\widehat{\mathbf{q}}}
\E_{\mathrm{sph}}(\widehat{\mathbf{q}}, \widehat{\mathbf{k}})
\bigr\|
\;\leq\;
C_\epsilon.
\]
\end{proposition}
\end{adjustwidth}

\begin{adjustwidth}{1.5em}{1.5em}
\begin{theorem}[Positive Definiteness on $\Sph^{d-1}$]
\label{thm:pd}
For all $d \geq 2$ and $\epsilon > 0$, the spherical \E-product
$k(x) = x^2/(C - 2x)$ with $x \in [-1,1]$ and $C = 2 + \epsilon$
is a positive-definite kernel on $\Sph^{d-1}$.
\end{theorem}
\end{adjustwidth}

\paragraph{Proof of Lemma~\ref{lem:complete-monotone}.}
\begin{proof}
Since $x \leq 1$ and $C = 2 + \epsilon$, we have
$C - 2x \geq C - 2 = \epsilon > 0$.
The function $g(y) = 1/y$ is completely monotone on $(0,\infty)$
(all derivatives alternate in sign:
$g^{(n)}(y) = (-1)^n n!/y^{n+1}$),
so Bernstein's theorem applies and yields
$1/(C-2x) = \int_0^\infty e^{-s(C-2x)}\,ds$.
\end{proof}

\paragraph{Proof of Proposition~\ref{prop:prf-unbiased}.}
\begin{proof}
\begin{align*}
\mathbb{E}\!\left[
\left\langle
\phi_{\text{PRF}}(\widehat{\mathbf{q}}; s),
\phi_{\text{PRF}}(\widehat{\mathbf{k}}; s)
\right\rangle
\right]
&=
\frac{1}{D}
\sum_{i=1}^D
\mathbb{E}\!\left[
\exp\!\left(
\sqrt{2s}\,
\boldsymbol{\omega}_i^\top
(\widehat{\mathbf{q}} + \widehat{\mathbf{k}})
- 2s
\right)
\right] \\
&=
e^{-2s}
\cdot
\exp\!\left(
s \|\widehat{\mathbf{q}} + \widehat{\mathbf{k}}\|^2
\right) \\
&=
e^{-2s}
\cdot
e^{s(2 + 2\widehat{\mathbf{q}}^\top \widehat{\mathbf{k}})} \\
&=
e^{2s\,\widehat{\mathbf{q}}^\top \widehat{\mathbf{k}}}.
\end{align*}
The unit-norm constraint
$\|\widehat{\mathbf{q}}\| = \|\widehat{\mathbf{k}}\| = 1$
is essential; otherwise, additional norm terms appear in
$\|\widehat{\mathbf{q}}+\widehat{\mathbf{k}}\|^2$.
\end{proof}

\paragraph{Proof of Proposition~\ref{prop:bounded}.}
\begin{proof}
Let $x = \widehat{\mathbf{q}}^\top \widehat{\mathbf{k}} \in [-1,1]$
and define $f(x)=x^2/(C-2x)$.
Since $C-2x \ge \epsilon > 0$ on $[-1,1]$, we have $f(x)\ge 0$.
Moreover,
\[
f'(x)=\frac{2x(C-x)}{(C-2x)^2}.
\]
On $[-1,1]$, the maximum of $f$ is attained at $x=1$, giving
$f(1)=1/\epsilon$.
\end{proof}

\paragraph{Proof of Proposition~\ref{prop:lipschitz}.}
\begin{proof}
Write
$\E_{\mathrm{sph}}(\widehat{\mathbf{q}},\widehat{\mathbf{k}})=f(x)$
with $x=\widehat{\mathbf{q}}^\top\widehat{\mathbf{k}}$
and $f(x)=x^2/(C-2x)$.
Differentiating gives
\[
f'(x)=\frac{2x(C-x)}{(C-2x)^2}.
\]
On $x\in[-1,1]$ with $C=2+\epsilon$, the denominator satisfies
$(C-2x)^2\ge \epsilon^2$ and the numerator is bounded, hence
$|f'(x)|\le C_\epsilon'$ for some constant depending only on $\epsilon$.
By the chain rule,
$\nabla_{\widehat{\mathbf{q}}}
\E_{\mathrm{sph}}(\widehat{\mathbf{q}},\widehat{\mathbf{k}})
=
f'(x)\nabla_{\widehat{\mathbf{q}}}x$
and
$\|\nabla_{\widehat{\mathbf{q}}}x\|
=
\|\widehat{\mathbf{k}}\|=1$
(or its tangent-space projection has norm $\le 1$),
giving the stated uniform bound.

If $\widehat{\mathbf{q}}$ is obtained by normalizing a pre-activation
vector $\mathbf{q}$ via
$\widehat{\mathbf{q}} = \mathbf{q}/\|\mathbf{q}\|$, then
\[
J(\mathbf{q})
=
\frac{1}{\|\mathbf{q}\|}
\bigl(
\mathbf{I} - \widehat{\mathbf{q}}\widehat{\mathbf{q}}^\top
\bigr),
\]
so
$\|J(\mathbf{q})\|_{\mathrm{op}} \le 1/\|\mathbf{q}\|$
wherever normalization is defined.
Thus the gradient w.r.t.\ $\mathbf{q}$ is controlled by the spherical
gradient and scales at most like $1/\|\mathbf{q}\|$.
\end{proof}

\paragraph{Proof of Theorem~\ref{thm:pd}.}
\begin{proof}
By Lemma~\ref{lem:complete-monotone},
\[
k(x)
=
\frac{x^2}{C-2x}
=
\int_0^\infty e^{-sC}\, x^2 e^{2sx} \, ds.
\]
For each $s \geq 0$, the integrand
$k_s(x)=x^2 e^{2sx}$ is PD because:
(i)
$x^2=(\widehat{\mathbf{q}}^\top\widehat{\mathbf{k}})^2$
is PD as a degree-2 polynomial kernel restriction;
(ii) for $s\ge 0$,
\[
e^{2s(\widehat{\mathbf{q}}^\top \widehat{\mathbf{k}})}
=
\sum_{n=0}^\infty
\frac{(2s)^n}{n!}
(\widehat{\mathbf{q}}^\top \widehat{\mathbf{k}})^n
\]
is a nonnegative linear combination of PD polynomial kernels, hence PD;
and (iii) products of PD kernels are PD.
Since the weight $e^{-sC}\ge 0$, the nonnegative mixture over $s$ is PD
by Lemma~\ref{lem:mixture-pd}.
\end{proof}

\section{Empirical Analysis of SLAY Properties}
\label{app:empirical-analysis}

This appendix provides empirical grounding for the theoretical claims in the main text. We present a unified analysis organized around three questions: (1)~Why does the spherical $\E$-kernel behave differently from softmax? (2)~Why is SLAY numerically stable when other polynomial approximations fail? (3)~How well does the quadrature approximation work in practice?

\subsection{Understanding the Spherical $\E$-Kernel}
\label{app:kernel-properties}

The spherical $\E$-kernel differs fundamentally from softmax in how it weighs token interactions. Figure~\ref{fig:kernel-comparison} compares the response curves: while softmax grows unboundedly as alignment approaches $x=1$, the spherical $\E$-kernel remains bounded by $1/\epsilon$. This self-regularization is a direct consequence of the inverse-distance denominator---highly aligned tokens are also close in representation space, and the two effects partially cancel.

\begin{figure}[htbp]
\centering
\includegraphics[width=.8\textwidth]{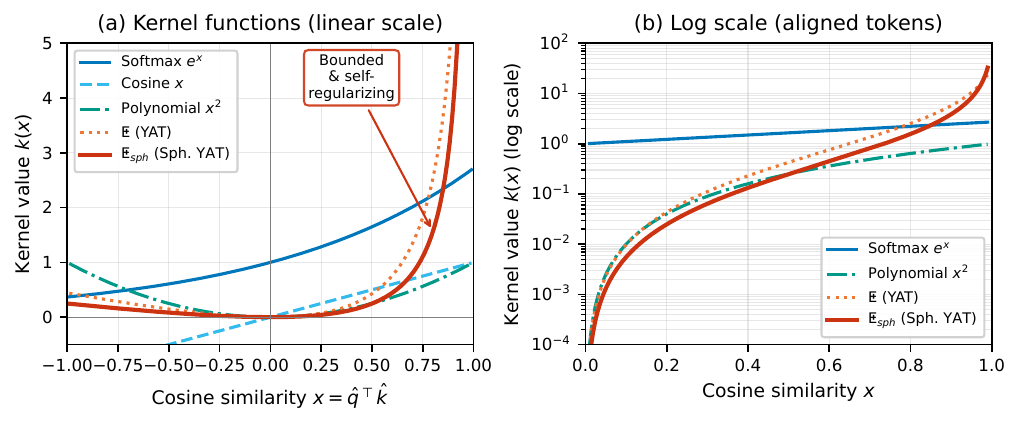}
\caption{Kernel response as a function of alignment $x=\hat{q}^\top\hat{k}$. The spherical $\E$-kernel is bounded and self-regularizing, unlike softmax which grows without bound.}
\label{fig:kernel-comparison}
\end{figure}

\begin{figure}[ht]
\centering

\includegraphics[width=.4\linewidth]{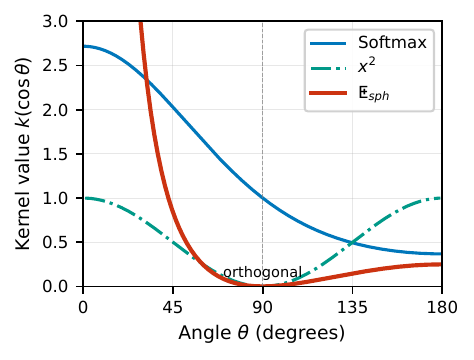}
\caption{Kernel response vs.\ angular distance.}
\label{fig:kernel-angle}

\end{figure}
This difference becomes even more pronounced when we examine angular discrimination. Figure~\ref{fig:kernel-angle} shows kernel response as a function of the angle between query and key. Spherical YAT exhibits sharp selectivity: it peaks dramatically for nearly-aligned tokens and drops to near-zero for orthogonal vectors. In contrast, softmax maintains appreciable weight even at 90°, effectively blurring the distinction between relevant and irrelevant tokens.

\begin{figure}[ht]
\centering
\includegraphics[width=.4\linewidth]{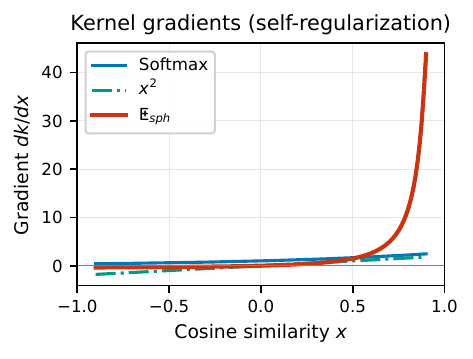}
\caption{Gradient magnitudes.}
\label{fig:kernel-derivatives}
\end{figure}
A natural concern is whether this sharper response leads to gradient instability. Figure~\ref{fig:kernel-derivatives} addresses this by comparing gradient magnitudes. While spherical YAT does exhibit larger gradients near perfect alignment ($x \approx 1$), these remain bounded and concentrated in a small region. For the vast majority of token pairs, gradients are well-behaved, enabling stable training.

\subsection{Numerical Stability: The Denominator Problem}
\label{app:stability-experiments}

A key claim of SLAY is that its random feature construction guarantees positive denominators, avoiding the catastrophic failures that plague signed polynomial approximations. This subsection provides direct empirical evidence for this claim.

Figure~\ref{fig:denominator-histogram} shows the distribution of denominator values (the sum $\sum_j \langle\phi(q_i), \phi(k_j)\rangle$ that normalizes attention) across different methods. The contrast is stark: SLAY and YAT variants produce strictly positive denominators in all samples, while TensorSketch and Random Maclaurin generate substantial fractions of negative values. When the denominator crosses zero, attention weights become undefined or flip sign, causing NaN gradients and training collapse.

\begin{figure}[htbp]
\centering
\includegraphics[width=.8\textwidth]{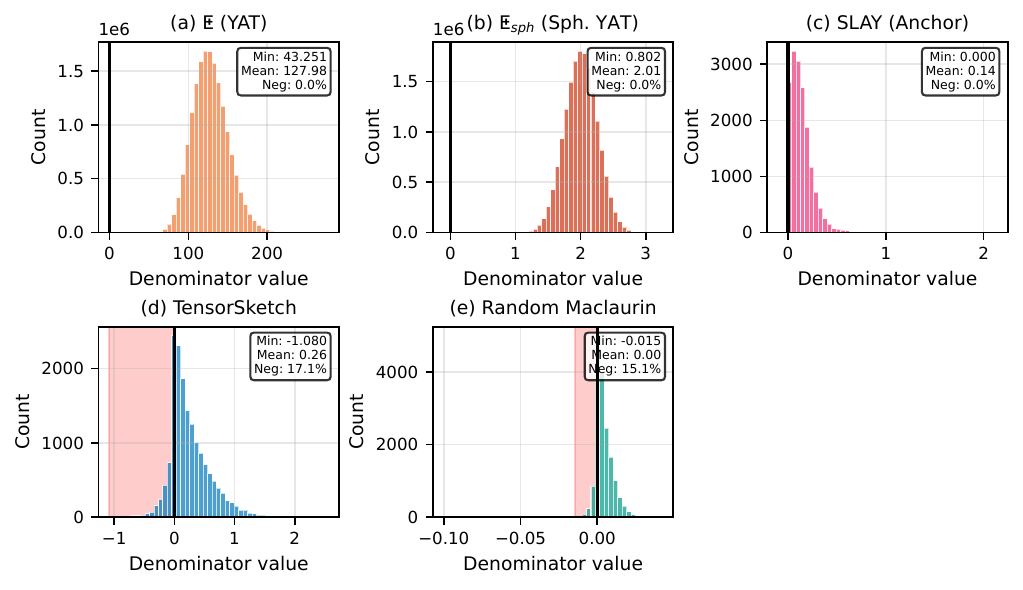}
\caption{Denominator value distributions. SLAY and YAT variants maintain strict positivity; signed polynomial methods (TensorSketch, Random Maclaurin) produce negative denominators that cause numerical instability.}
\label{fig:denominator-histogram}
\end{figure}

\begin{figure}[ht]
\centering

\includegraphics[width=.4\linewidth]{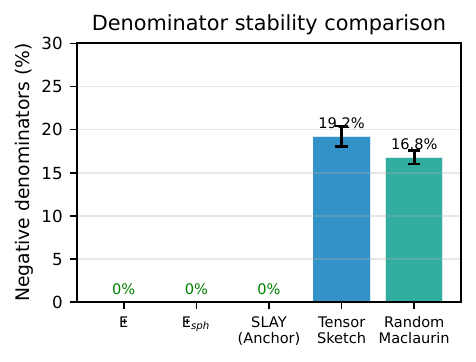}
\caption{ Stability across seeds.}
\label{fig:denominator-stability}

\end{figure}
One might wonder whether SLAY's stability is merely a lucky outcome for particular random seeds. Figure~\ref{fig:denominator-stability} dispels this concern by showing stability across multiple random initializations. SLAY's positivity guarantee is deterministic---it follows from the construction, not from chance---and the experiments confirm consistent behavior regardless of seed.

\subsection{Quadrature Approximation Quality}
\label{app:quadrature-details}

The integral representation of the spherical $\E$-kernel (Eq.~\ref{eq:integral}) is discretized via Gauss-Laguerre quadrature. A central practical question is: how many nodes $R$ are needed for acceptable accuracy?

Figure~\ref{fig:quadrature-convergence} shows that error decreases exponentially with $R$---the hallmark of Gaussian quadrature for smooth integrands. This rapid convergence justifies our default choice of $R=3$, which achieves low error while contributing negligible overhead to total compute.

\begin{figure}[htbp]
\centering
\includegraphics[width=.8\textwidth]{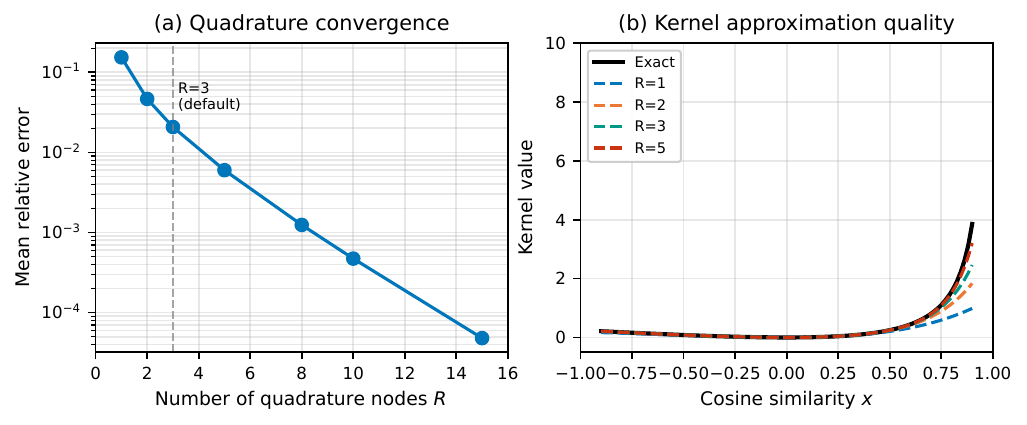}
\caption{Quadrature error vs.\ number of nodes $R$. Exponential convergence enables small $R$ in practice.}
\label{fig:quadrature-convergence}
\end{figure}

\begin{figure}[ht]
\centering

\includegraphics[width=0.4\linewidth]{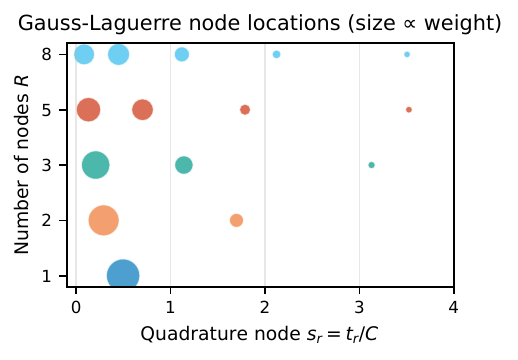}
\caption{Gauss-Laguerre nodes.}
\label{fig:quadrature-nodes}

\end{figure}
Why does small $R$ suffice? The Gauss-Laguerre nodes are not uniformly spaced: lower-indexed nodes sit closer to zero and carry substantially larger weights (Figure~\ref{fig:quadrature-nodes}).

\begin{figure}[ht]
\centering

\includegraphics[width=.4\linewidth]{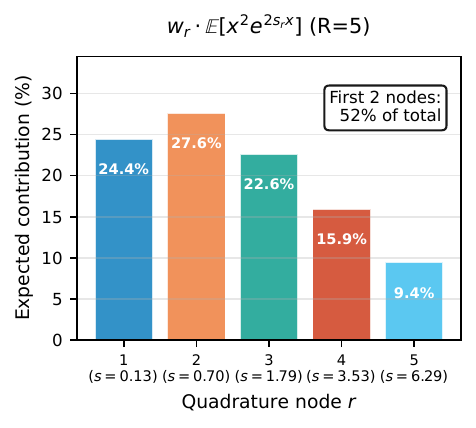}
\caption{Expected contribution.}
\label{fig:expected-contribution}

\end{figure}
As a result, the first few nodes contribute the majority of the integral approximation (Figure~\ref{fig:expected-contribution}).

This concentration persists across different values of the alignment variable $x$ (Figure~\ref{fig:quadrature-contributions}).

\begin{figure}[htbp]
\centering
\includegraphics[width=.8\textwidth]{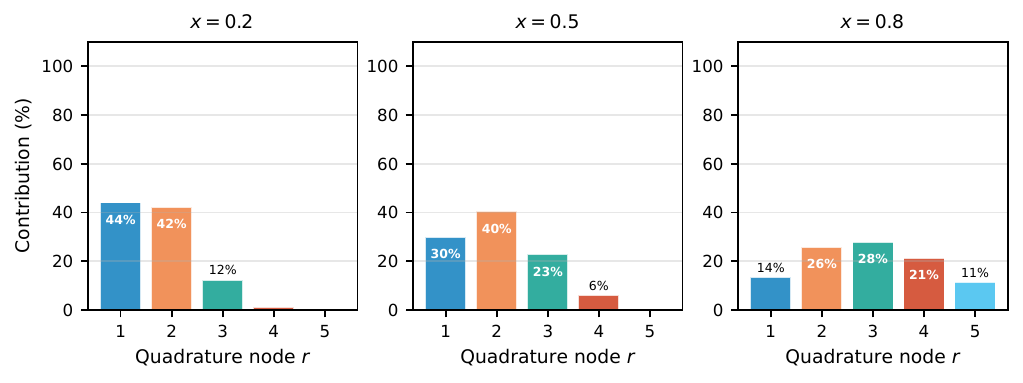}
\caption{Node contributions at different $x$ values. First nodes dominate across all alignments.}
\label{fig:quadrature-contributions}
\end{figure}

Figure~\ref{fig:approximation-quality} shows the end-to-end kernel reconstruction: SLAY closely matches the quadrature-only baseline, confirming that the random feature component adds minimal additional error. Finally, Figure~\ref{fig:error-features} demonstrates that SLAY's approximation error is stable across feature budgets, with the quadrature discretization---not the random features---being the dominant source of approximation error.

\begin{figure}[htbp]
\centering
\includegraphics[width=.8\textwidth]{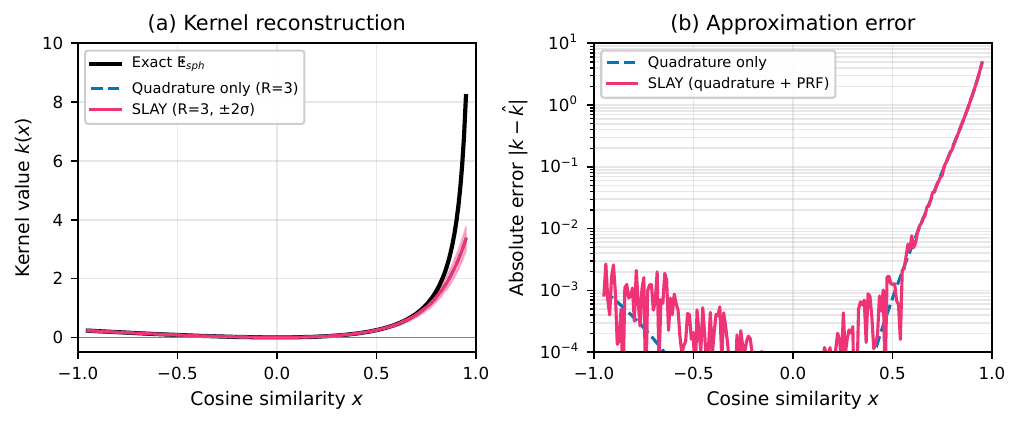}
\caption{Kernel reconstruction quality. SLAY closely matches pure quadrature.}
\label{fig:approximation-quality}
\end{figure}

\begin{figure}[ht]
\centering

\includegraphics[width=.4\linewidth]{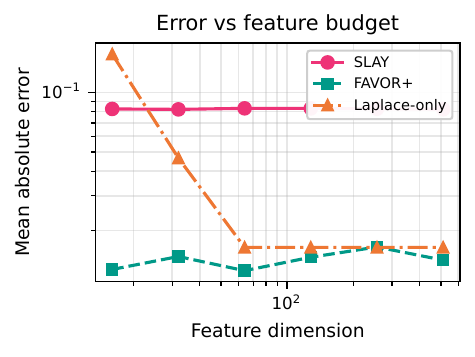}
\caption{Error vs.\ features.}
\label{fig:error-features}

\end{figure}
Finally, Figure~\ref{fig:error-features} demonstrates that SLAY's approximation error is stable across feature budgets, with the quadrature discretization---not the random features---being the dominant source of approximation error.

\subsection{Attention Selectivity and Patterns}
\label{app:attention-analysis}

The kernel-level differences documented above translate into measurable differences in attention behavior. This subsection examines how attention weight is distributed across tokens.

\begin{figure}[ht]
\centering
\includegraphics[width=.4\linewidth]{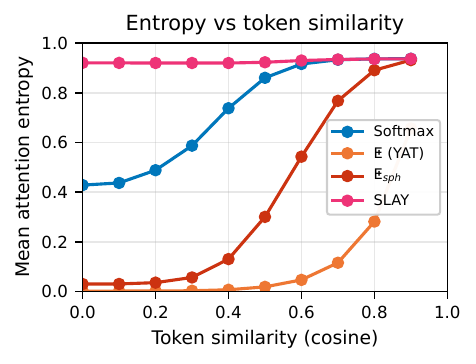}
\caption{Entropy vs.\ similarity.}
\label{fig:entropy-similarity}
\end{figure}
Figure~\ref{fig:entropy-similarity} plots attention entropy---a measure of how diffuse or concentrated the attention distribution is---as a function of token similarity. At low similarity (random embeddings), YAT mechanisms exhibit dramatically lower entropy than softmax, indicating highly selective attention that focuses on the few genuinely relevant tokens rather than spreading weight broadly. As similarity increases, all methods converge, as expected when most tokens become relevant.

Figure~\ref{fig:attention-entropy} provides another view of the same phenomenon, showing the distribution of entropy values. Finally, Figure~\ref{fig:attention-patterns} visualizes representative attention matrices: YAT and SLAY produce visibly more concentrated patterns compared to the broader, more uniform patterns of softmax.

\begin{figure}[htbp]
\centering
\includegraphics[width=.8\textwidth]{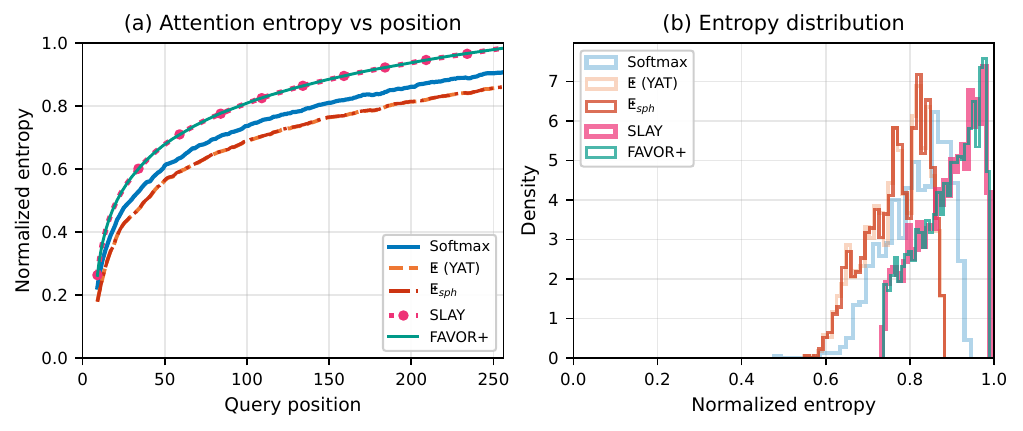}
\caption{Distribution of attention entropy values across mechanisms.}
\label{fig:attention-entropy}
\end{figure}

\begin{figure}[htbp]
\centering
\includegraphics[width=.8\textwidth]{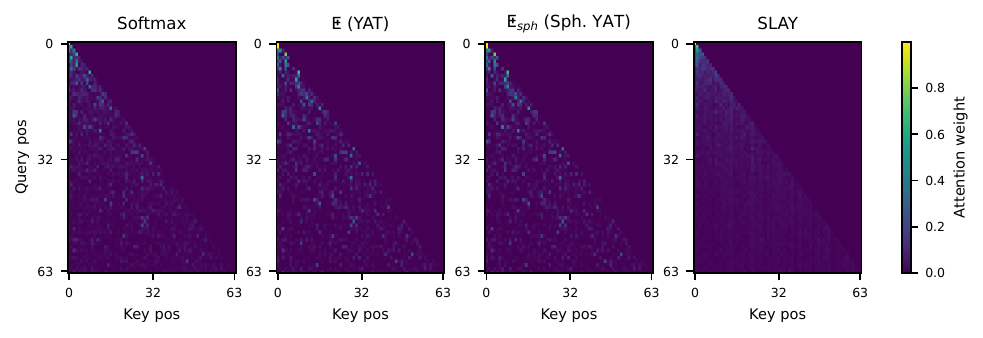}
\caption{Attention pattern visualization. YAT and SLAY show more concentrated patterns; softmax is more diffuse.}
\label{fig:attention-patterns}
\end{figure}

A natural question is whether SLAY's approximation preserves the attention behavior of exact spherical YAT. Figure~\ref{fig:attention-comparison} compares attention outputs directly, showing high correlation between the exact and approximate methods.

\begin{figure}[htbp]
\centering
\includegraphics[width=.8\textwidth]{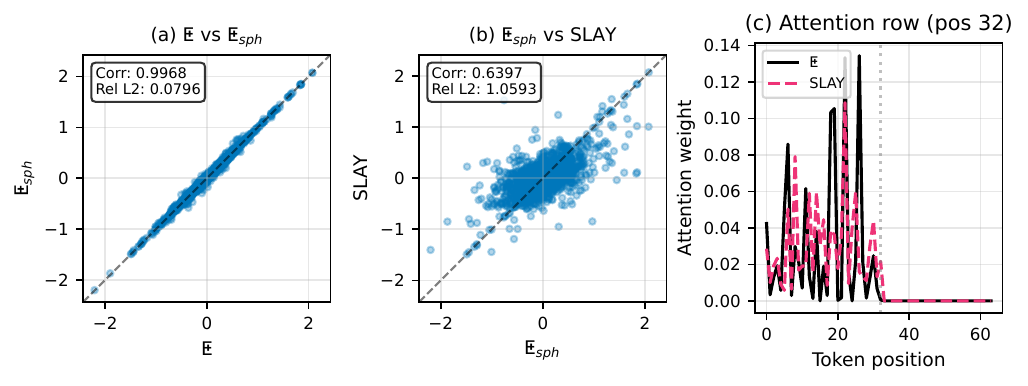}
\caption{Attention output comparison between exact and SLAY-approximated methods.}
\label{fig:attention-comparison}
\end{figure}

\subsection{Spherical Geometry Visualization}
\label{app:spherical-viz}

The spherical constraint enables direct visualization of attention behavior on the unit sphere $\Sph^{d-1}$. With a query fixed at the north pole, we can examine how different kernels weight keys distributed across the sphere.

Figure~\ref{fig:spherical-heatmap} shows 3D heatmaps of attention weight. YAT and SLAY concentrate attention strongly around the query direction; softmax distributes weight more uniformly, assigning appreciable weight even to keys far from the query. Figure~\ref{fig:spherical-polar} presents the same data as 2D polar profiles, making the difference in sharpness immediately apparent.

\begin{figure}[htbp]
\centering
\includegraphics[width=.8\textwidth]{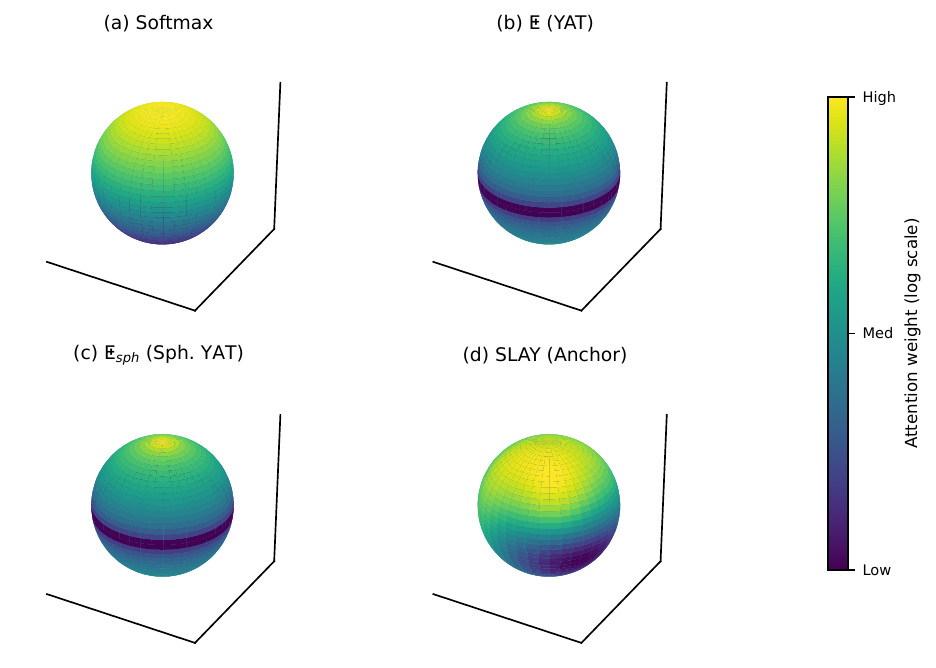}
\caption{3D spherical attention heatmap. YAT and SLAY concentrate weight around the query; softmax is more diffuse.}
\label{fig:spherical-heatmap}
\end{figure}

\begin{figure}[htbp]
\centering
\includegraphics[width=.8\textwidth]{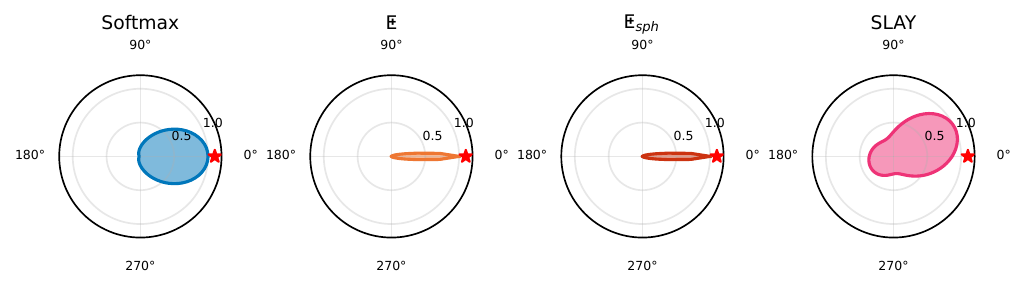}
\caption{Polar attention profile. Sharper drop-off indicates more geometry-aware attention.}
\label{fig:spherical-polar}
\end{figure}

\subsection{Computational Scaling}
\label{app:scaling-analysis}

Finally, we verify that SLAY achieves the promised linear scaling. Figure~\ref{fig:scaling-exact} compares latency, memory, and throughput as a function of sequence length. Quadratic methods (standard attention, exact YAT) scale poorly and run out of memory beyond 16K tokens. SLAY, along with other linear attention methods, scales to 131K tokens and beyond while maintaining high throughput.

\begin{figure}[htbp]
\centering
\includegraphics[width=.8\textwidth]{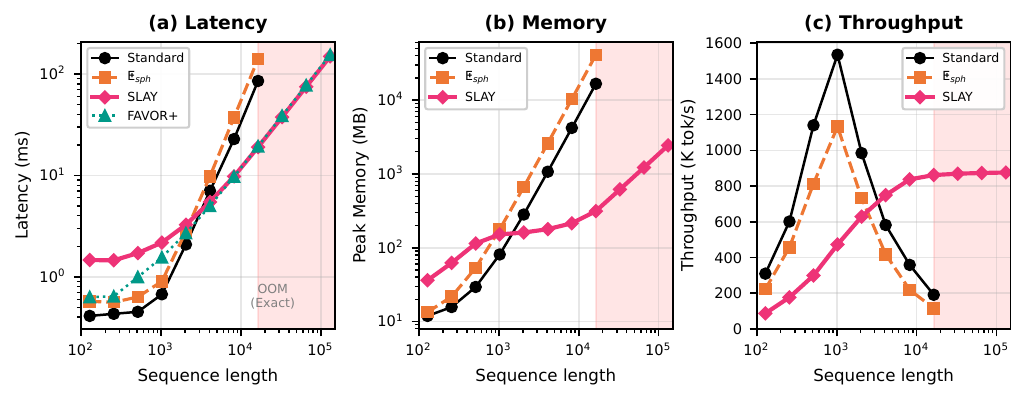}
\caption{Computational scaling comparison. SLAY maintains linear complexity up to 131K tokens where quadratic methods fail.}
\label{fig:scaling-exact}
\end{figure}

\paragraph{Summary.}
Together, these experiments validate the core claims: the spherical $\E$-kernel provides sharper, more selective attention than softmax; SLAY's construction guarantees positive denominators; the quadrature approximation converges rapidly; and the resulting mechanism scales linearly in sequence length.

\textbf{Attention Mechanisms}. We compare the following mechanisms (Table \ref{tab:attention_params}). Note that $\epsilon$ values ensure numerical stability and vary by mechanism.

\begin{table}[H]
\centering
\caption{Attention mechanisms and their configurations.}
\label{tab:attention_params}
\small
\begin{tabular}{llcl}
\toprule
\textbf{Method} & \textbf{Type} & $\bm{\epsilon}$ & \textbf{Parameters / Notes} \\
\midrule
Standard & Softmax & - & Exact, quadratic cost. \\
Linear & ELU+1 & $10^{-6}$ & Feature map $\phi(x) = \text{elu}(x)+1$. \\
Performer & FAVOR+ & - & $M=64$ random features (ReLU). \\
Yat & Exact & $10^{-3}$ & Exact Yat-kernel. \\
Yat Spherical & Exact & $10^{-3}$ & Exact Spherical Yat-kernel. \\
\textbf{SLAY} & Linear & $10^{-3}$ & $M_{\text{RFF}}=64, M_{\text{PRF}}=16, M_{\text{Poly}}=8$. \\
\bottomrule
\end{tabular}
\end{table}

\end{document}